\documentclass[sigconf]{acmart}
\AtBeginDocument{%
  }

\copyrightyear{2026}
\acmYear{2026}
\setcopyright{cc}
\setcctype{by}
\acmConference[KDD '26]{Proceedings of the 32nd ACM SIGKDD Conference on Knowledge Discovery and Data Mining V.2}{August 09--13, 2026}{Jeju Island, Republic of Korea}
\acmBooktitle{Proceedings of the 32nd ACM SIGKDD Conference on Knowledge Discovery and Data Mining V.2 (KDD '26), August 09--13, 2026, Jeju Island, Republic of Korea}
\acmDOI{10.1145/3770855.3818112}
\acmISBN{979-8-4007-2259-2/2026/08}

\usepackage[utf8]{inputenc} 
\usepackage[T1]{fontenc}    
\usepackage{hyperref}       
\usepackage{url}            
\usepackage{booktabs}       
\usepackage{amsfonts}       
\usepackage{nicefrac}       
\usepackage{microtype}      
\usepackage{xcolor}         
\usepackage{bbm}            
\usepackage{amsmath, amsthm}
\newcommand{\indep}{\perp \!\!\! \perp}
\usepackage{mathtools}
\usepackage{algorithm}
\usepackage{algpseudocode}
\usepackage{subcaption}
\usepackage{caption}
\usepackage{enumitem}
\usepackage{graphicx}
\usepackage{tabularx}
\usepackage{makecell}     
\usepackage{adjustbox}    
\usepackage{makecell} 
\usepackage{tablefootnote}

\numberwithin{equation}{section}
\theoremstyle{plain}
\newtheorem{assumption}{Assumption}

\newtheorem{constraint}{Constraint}

\begin{document}

\title{A Practical Upper Bound on Selection Bias Effects in Medical Prediction Models}

\author{Kara Liu}
\affiliation{%
  \institution{Stanford University}
  \city{Stanford}
  \state{California}
  \country{USA}
}
\author{Maggie Wang}
\affiliation{%
  \institution{Stanford University}
  \city{Stanford}
  \state{California}
  \country{USA}
  }

\author{Russ B. Altman}
\affiliation{%
  \institution{Stanford University}
  \city{Stanford}
  \state{California}
  \country{USA}}
\renewcommand{\shortauthors}{Liu et al.}

\begin{abstract}
Selection bias is a common and often unavoidable aspect of real-world data that challenges the generalizability of machine learning models. When models trained on biased data are deployed in the broader target population, poor model generalization may lead to real harm, particularly in high-risk settings such as healthcare. This risk highlights the need for practitioners to reliably assess model generalizability prior to deployment. However, existing methods for predicting model performance rely on unrealistic access to the target distribution or knowledge of the selection mechanism causing bias. To address these limitations, we propose a novel upper bound on the worst-case model performance on the target population under the realistic setting where the selection mechanism and the target population data are only partially observed. We demonstrate the validity and practical utility of our method through experiments on fully synthetic data, semi-synthetic data derived from the All of Us Research Program, and real-world selection bias in MIMIC-IV. Our work offers a principled and practical tool to estimate the impact of selection bias in an otherwise intractable setting, thereby enabling practitioners to build safer and more generalizable models in healthcare and beyond. We release our code for public use at \url{https://github.com/kara-liu/selection_gap_est/}.

\end{abstract}



\maketitle

\begin{figure*}[!t]
    \centering
    \includegraphics[width=\textwidth]{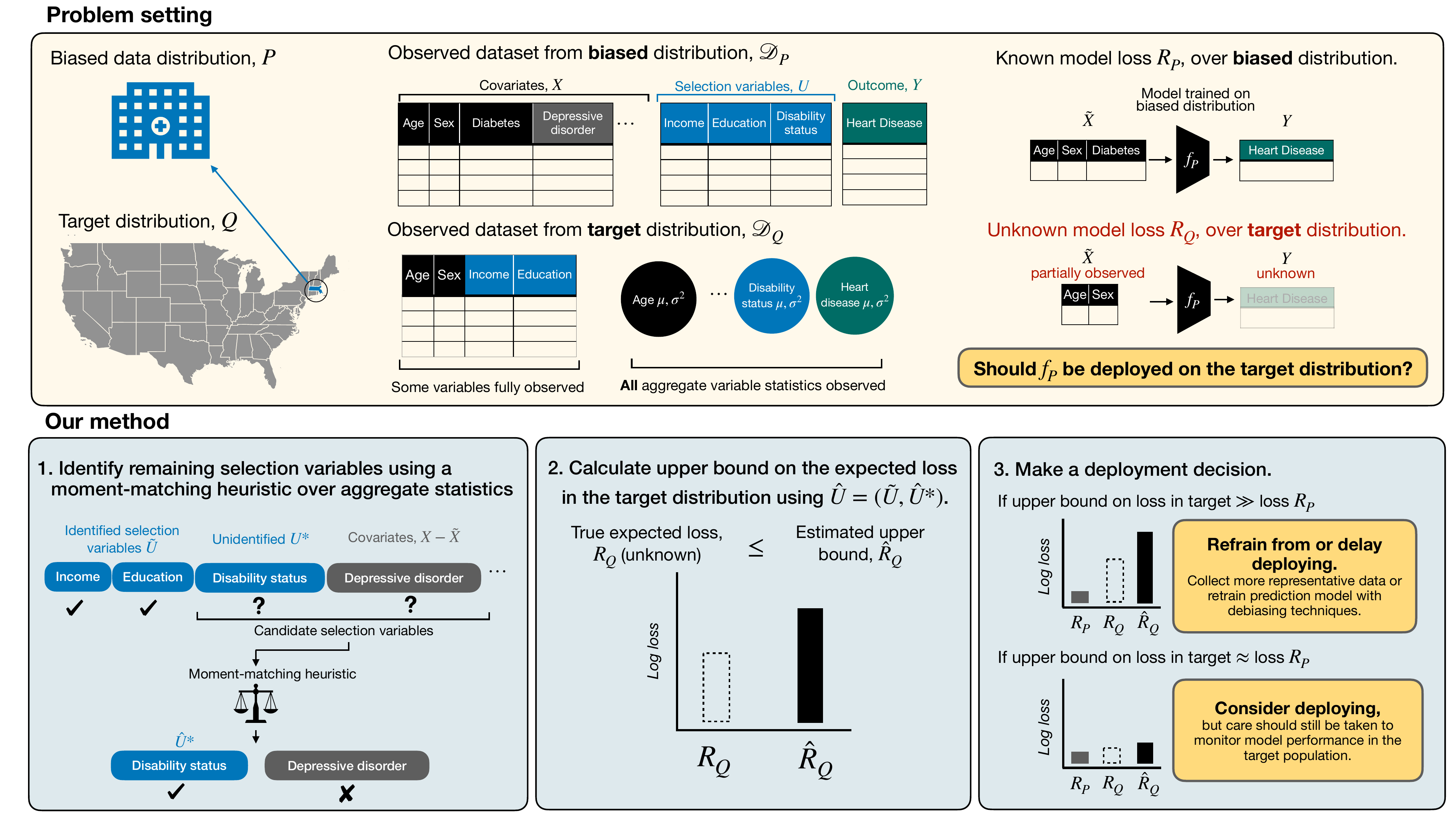}
    \caption{Pipeline illustrating how our method estimates a generalization bound under limited observation of the target-data, thereby enabling better informed decisions of model deployment.}
    \label{fig:myfig}
\end{figure*}

\section{Introduction}

As machine learning models are increasingly deployed in real-world settings, it is imperative to understand how their performance may be affected by selection bias. When models are trained on data that represents only a subset of the population, their failure to generalize to the broader target population could inflict real harm, often in ways that reify discrimination against underrepresented groups \cite{goetz2024generalization,ahmad2025bias,han2025addressing}. This risk is particularly acute in medical applications, where the data used to support high-stakes decisions are often heavily skewed by selection bias \citep{graham,bibbins,covid}. For instance, selection bias in biobank data, electronic health records, and randomized controlled trials has led to biased estimates in genome-wide association studies \citep{ukbbgwas}, analyses of COVID-19 risk factors \citep{covid}, the prediction of sepsis in hospitals \citep{sepsis}, and drug dosage recommendations that are suboptimal for non-Caucasian populations \citep{warfarin}.

To support safe deployment under selection bias, machine learning model developers must be able to assess how a model trained on biased data will perform on its intended target population, in a way that is both practical and grounded in the underlying selection mechanism. As a motivating example, the national deployment of the Epic Sepsis Model, which was trained on data from just three healthcare systems, was later criticized for poor generalization \cite{sepsis,lyons2023factors}. These issues might have been prevented had developers been able to foresee how the model would perform on the broader U.S. population.

Selection bias has been studied in other disciplines, notably in the context of causal effect estimation \citep{t1t2,pearl} or under a limited set of selection mechanisms in domain adaptation \cite{da1,da2}. However, existing methods that estimate model generalizability often rely on unrealistic assumptions. For instance, density ratio estimation requires access to the full target distribution \cite{da1,da2,da3}, while inverse probability of participation weighting \citep{ipw1,clf1,bonander2019participation,ipw4} assumes full observability of the features causal of selection. In practice, however, access to target data is often limited, as collecting fully representative samples -- for example, medical records from the entire U.S. population -- can be prohibitively expensive, logistically infeasible, and may conflict with privacy regulations. Moreover, model developers rarely observe the explicit underlying selection mechanism, given the complexity in delineating the causes of study participation or health care utilization. As a result, there are currently few practical and principled approaches for practitioners to audit models for selection bias.

\paragraph{Contributions}
In this work, we propose a tractable estimate of a prediction model's worst-case performance on an unobserved target population. Our paper offers three important contributions:

\begin{enumerate}
    \item We consider a more realistic setting that requires only partial observability of the selection mechanism and target distribution. Target data sources, such as national registries and census microdata, often provide full joint coverage of a few basic sociodemographic variables, along with marginal summary statistics such as first and second moments \cite{ipw4,ukbbgwas,elliott2017inference,giorgi2022correcting,ipw1}. Motivated by these real-world constraints, our method assumes: (i) only a subset of variables causing selection has been identified and fully observed and (ii) access to marginal summary statistics for all variables. The feasibility of (i) is further supported by prior work showing that sociodemographic factors such as age, sex, and income, which are often observable in target datasets, are also likely drivers of medical selection bias \cite{ukbbgwas,silva2015assessment,saphner2021clinical}.
    \item We propose a novel upper bound on a prediction model's expected performance on a target dataset. Our bound explicitly models the selection process without requiring full knowledge of the underlying mechanism. To the best of our knowledge, existing methods are incapable of producing such a bound under the realistic constraints outlined in (1).
    \item To render our bound tractable for practitioners auditing models prior to deployment, we propose two heuristics: first, an algorithm that identifies the remaining selection variables, and second, diagnostic techniques for assessing the assumptions underlying our bound.
\end{enumerate}

We evaluate our method in three data settings: (i) simulated selection bias in fully synthetic data, (ii) simulated selection bias in clinical data from the All of Us Research Program \citep{aou}, and (iii) real selection bias in MIMIC-IV \citep{johnson2024mimiciv}. Across these experiments, we show that our proposed bound is empirically tight and robust, even under realistic data observability constraints. Finally, we provide guidance and diagnostic tools to help practitioners apply and interpret our bound estimate.

\section{Related Works}

\subsection{Selection Bias in Causal Inference}

Selection bias has been broadly studied across disciplines including econometrics \citep{heckman, heckmanmethod, heckman2, econ}, causal inference \citep{pearl}, epidemiology \citep{hernan, t1t2, infante, smith}, statistics \citep{stats1, mansiki}, social science \citep{berk, social}, and clinical informatics \citep{cer,clinical}. In these fields, the focus has largely been on causal effect estimation where bias arises when generalizing estimates from a non-representative study sample to a target population. Many debiasing methods have been developed to address this issue, including g-formula adjustment \citep{pearl, lesko}, propensity score weighting \citep{lesko, ipw2000, whatif}, and Heckman correction \citep{heckmanmethod}. However, these methods are not directly applicable to our setting of assessing the generalizability of prediction models. Furthermore, these methods often rely on unrealistic assumptions, such as full observability of the causes of selection.

Another related line of work has estimated robustness to \textit{unobserved} selection or confounding via sensitivity analysis and partial identification. Motivated by Rosenbaum-style sensitivity analysis \citep{RosenbaumSensitivity}, \citet{ZhaoSensitivity} proposed bootstrapped confidence intervals based on a marginal sensitivity model that bounds the unknown selection probability. Other approaches avoid modeling the unobserved confounding mechanism by providing explicit sensitivity parameters to bound confounder-treatment and confounder-outcome associations \citep{ding2016sensitivity} or by using partial identification to propose worst-case bounds, as with Manski-style bounds \citep{manski2003partial,mansiki}. While these approaches make minimal assumptions on the underlying selection mechanism, they rely on ad hoc specification of the sensitivity parameters which can lead to overly loose bounds in practice.  

Selection bias may also be framed as a case of estimation under missingness-not-at-random of the covariates, outcome, or both. When data are only observed conditioned on a positive outcome, \citep{tchetgen2017general, little2019statistical,sun2018semiparametric} proposed methods to identify causal effects by leveraging instrumental variables. However, these works assume full observation of a valid instrument and covariates in the target distribution.

\subsection{Domain Adaptation}
In machine learning, selection bias has most often been studied through the lens of domain adaptation where a model trained on a source domain (or distribution) must generalize to a target domain \citep{da1,da3}. Domain adaptation can be approached by learning domain-invariant representations \citep{dann}; reweighting samples from the source domain \citep{da1,da2,da3} using weights obtained by probabilistic classification \citep{clf1}, moment matching \citep{kmm,kmmog}, or density matching \citep{kliep}; and through distributionally robust optimization \citep{dro1,dro2}, which learns a model that minimizes worst-case performance over ``uncertainty sets" of the observed data. However, as in the case of causal inference methods, domain adaptation methods typically rely on the variables causing the data shift to be observable from the target distribution, which is an unrealistic assumption in real-world settings. Additionally, these methods are used \textit{during} model training to improve out-of-distribution performance, rather than \textit{after} model training to audit model generalizability.

\subsection{Generalization Bounds}
Another relevant area of work involves estimating model generalization bounds under distribution shift. \citet{cortes2010learning} derive generalization bounds that depend on the stability of the associated importance weights, and \citet{ben1, ben2} propose an upper bound on the generalization error of a model trained on samples from distribution $P$ when applied to distribution $Q$, using $\mathcal{H}$-divergence to characterize the distance between $P$ and $Q$. Similar bounds have also been proposed using Wasserstein distance \citep{wasser} and $f$-divergence \citep{fdiver}. As with domain adaptation methods, these worst-case bounds are intractable without strong assumptions on what is observable.

\section{Method}
In Sections \ref{sec:methods:prob-form}, we introduce notation, key assumptions, and observability constraints. In Section \ref{sec:methods:bound}, we introduce our method for calculating the generalization upper bound. In Section \ref{methods:impl}, we outline how our method is practically implemented via a heuristic algorithm that nominates the remaining selection variables. We outline our method in Figure \ref{fig:myfig}.

\subsection{Problem Formulation}\label{sec:methods:prob-form}
We denote an upper-case letter $V$ as a single or set of random variables, the variable's set space as $\mathcal{V}$, and the observation in that space as the lower-case $v$. We use the notation $\mathbb{E}_{Q}[\cdot]$ to describe the expectation with respect to sampling from a distribution $Q$. We let $V_j^{(i)}$ denote the $j$-th variable and $i$-th unit of a multivariate sample. 

Similar to the notation of \cite{roshni,zadrozny2004learning}, let $Q$ be the target distribution over the variables $(X, U, Y, S)$, where $S \in \{0,1\}$ is a binary selection indicator, $U \in \mathcal{U} \subset \mathbb{R}^k$ are the variables causal of selection, $Y \in \mathcal{Y} \subset \mathbb{R}$ is the outcome, and $X \in \mathcal{X} \subset \mathbb{R}^d$ are all other covariates. We assume the biased\footnote{In other domains, $P$ is also called the source, sample, or training distribution.} distribution $P$ is generated by selective sampling from $Q$ such that the $(X, Y)$-marginal of $Q$ conditioned on $S = 1$ is exactly the distribution $P$. That is, denoting $p(V)$ as the probability distribution of any variable $V \in (X, Y)$ under $Q$, then $p(V \mid S=1)$ is the variable's distribution under $P$, where we assume $P$ and $Q$ have common support. Furthermore, we assume that the selection indicator $S$ is defined as a probabilistic function of selection variables $U$, and thus $S$ is conditionally independent of $(X,Y)$ given $U$. We summarize the variables in Table \ref{tab:variables}.

\begin{table}[ht!]
    \caption{Notation for our problem setting.}
    \centering

    \resizebox{\columnwidth}{!}{
    \begin{tabular}{p{0.06\columnwidth}  p{0.63\columnwidth} p{0.41\columnwidth}}

    \toprule
    Var. & Description & Example  \\
    \midrule
    $Q$ & Target distribution over $(X, U, Y, S)$ & U.S. population distribution\\
    $P$ & Biased distribution over $(X, U, Y)$ generated by selective sampling from $Q$ & Single hospital distribution \\
    $S$ & Binary selection indicator & If patient in the U.S. attended single hospital \\ 
    $U$ & Variables causal of selection into biased distribution $P$ & Income, education, disability status \\
      $\tilde{U}$ & Variables $\subseteq U$ observed in both $P$ and $Q$ & Income, education  \\
     $U^*$ & Variables $:= U \setminus \tilde{U}$ not observed under $Q$ & Disability status \\
    $X$ & All covariates non-overlapping\tablefootnote{In practice, $\tilde{X}$ and $U$ may overlap, provided all covariates used for prediction are observed in $Q$.} with $U$ & Age, sex, diabetes, depressive disorder
 \\
 $\tilde{X}$ & Covariates $\subseteq X$ used to predict $Y$ & Age, sex, diabetes \\
    $Y$ & Outcome & Heart disease \\
    \bottomrule
    \end{tabular}
    }
    \label{tab:variables}
\end{table}

\begin{assumption}[Conditional Independence]\label{assump:cond-indep}
        $X, Y \indep S \mid U$
\end{assumption}

\begin{assumption}[Common Support]\label{ass:support}
    $\forall X \in \mathcal{X}, U \in \mathcal{U}, Y \in \mathcal{Y}$, if $p(X,U,Y) > 0$, then $p(X,U,Y \mid S=1) > 0$. 
\end{assumption}

Let $f_P: \tilde{\mathcal{X}} \rightarrow \mathcal{Y}$ denote a prediction model trained under the biased distribution $P$ to minimize the expected loss $R_P = \mathbb{E}_{P}[\ell (f_P(\tilde{X}), Y)]$, where the subvector $\tilde{X}$ of $X$ are the prediction features and $\ell: \mathcal{Y} \times \mathcal{Y} \rightarrow \mathbb{R}^+$ is a non-negative loss function.
\begin{assumption}[Non-negative Loss]\label{assump:pos-loss}
    $\ell (f_P(\tilde{X}), Y) \geq 0$ for all $\tilde{X} \in \mathcal{\tilde{X}}$ and $Y \in \mathcal{Y}$.
\end{assumption}

To evaluate model generalization on the target distribution $Q$, we consider estimating the expected loss $R_Q = \mathbb{E}_{Q} [\ell (f_P(\tilde{X}), Y)]$ under the relaxed but more challenging setting of limited $Q$ observability. Specifically, we assume identification and full observation of a subset $\tilde{U}$ of the selection variables, as well as basic summary statistics for all candidate selection variables. We denote the unknown selection variables as $U^*:= U \setminus \tilde{U}$.

\begin{constraint}[Absence of Samples $(\tilde{X}, Y) \sim Q$]\label{ass:no_samples} The variables $(\tilde{X}, Y)$ needed to measure the performance of $f_P$ are unobserved in $Q$.
\end{constraint}

\begin{constraint}[Observation of Partial $\tilde{U}$ from $Q$]\label{ass:utilde}
    We identify a non-empty subset of selection variables $\tilde{U} \subseteq U$ and observe $\tilde{U}$ under $Q$. 
\end{constraint}
    
\begin{constraint}[Observation of $\mu,\sigma^2$ from $Q$]\label{ass:mu_sigma}
    For all variables $V \in (X, U)$, we observe its mean $\mu_Q(V) = \mathbb{E}_Q[V]$ and variance $\sigma^2_Q(V) = \mathbb{E}_Q[(V - \mu_Q(V))^2]$.
\end{constraint}

Finally, we define the observed dataset $\mathcal{D}_P$ = $\{(X^{(i)}, U^{(i)}, Y^{(i)}, S^{(i)}=1)\}_{i=1}^n$ drawn independently from the distribution $P$, on which the prediction model $f_P$ is trained. From $Q$, we observe the summary statistics defined in Constraint \ref{ass:mu_sigma}, as well as $\mathcal{D}_{Q}$ =$\{\tilde{U}^{(i)}\}_{i=1}^m$.

\subsection{An Upper Bound of ${R}_Q$}\label{sec:methods:bound}
Under the Constraints \ref{ass:no_samples} - \ref{ass:mu_sigma} of limited target distribution observability, \textit{the expected loss $R_Q$ is intractable using existing methods}. To address this gap, we propose an upper bound $\hat{R}_Q$ that bounds the performance of the model $f_P$ on the target distribution $Q$. The construction of the bound relies on both the fully observed $\tilde{U}$ and identification of the remaining selection variables $U^*$.

\begin{theorem}[Upper Bound $\hat{R}_Q$]\label{thm:bound}
    Under Assumptions \ref{assump:cond-indep}, \ref{ass:support}, and \ref{assump:pos-loss},
    \begin{align*}
    R_Q &\leq \hat{R}_Q := \mathbb{E}_P
        \left[ w(\tilde{U}) \cdot \phi(\tilde{X}, Y, \mathcal{U}^*) \cdot \ell (f_P(\tilde{X}), Y)\right] \\
    &w(\tilde{U}) :=  \frac{p(\tilde{U})}{p(\tilde{U} \mid S=1)} \notag \\
    &\phi(\tilde{X}, Y, \mathcal{U}^*) :=  \frac{\max_{u^* \in \mathcal{U^*}}p(\tilde{X}, Y \mid u^*, \tilde{U}, S=1)}{p(\tilde{X}, Y \mid \tilde{U}, S=1)} \notag
    \end{align*}
    where $u^*$ is an observation in the subspace of $\mathcal{U}^*$.
\end{theorem}

The full proof is in Appendix \ref{appendix:proofs:ub}. To build intuition for the proof, observe that under Assumption \ref{ass:support}, $R_Q$ can be expressed as a reweighted expectation over $P$, where the weights correspond to the density ratio of $(\tilde{X},Y,\tilde{U})$ in $P$ and $Q$. Although the marginal $p(\tilde{U})$ under $Q$ is known, the conditional distribution $p(\tilde{X},Y \mid \tilde{U})$ is not. The key insight is that this conditional distribution can be upper bounded by $p(\tilde{X},Y \mid U)$ under $Q$. Then, invoking conditional independence in Assumption \ref{assump:cond-indep}, we can replace this unknown distribution with the known marginal  $p(\tilde{X},Y \mid U, S = 1)$ under $P$.

We define the \textit{true generalization gap} $R_Q - R_P$ as the loss increase when evaluating $f_P$ on $Q$ versus $P$, the \textit{upper bound}\footnote{Estimating the upper bound is appropriate when higher loss $\ell$ indicates worse performance (such as in the case of logloss or Brier score); if lower loss $\ell$ indicates worse performance (such as with precision or accuracy), the lower bound $\hat{R}_Q \le R_Q$ can be constructed by taking the minimum instead of the maximum over ${U}^*$.} \textit{generalization gap} as $\hat{R}_Q - R_P$, and the \textit{bound error} as $\hat{R}_Q - R_Q$.

\subsection{Practical Bound Estimation}\label{methods:impl}
We next outline how to estimate our proposed bound. In Section \ref{sec:methods:u_pred}, we propose a heuristic algorithm to identify the remaining selection variables $U^*$. In Sections \ref{method:density_est} and \ref{method:ass_tests}, we outline our approach to density estimation and potential assumption violations in finite-sample settings. The pseudocode of our bound estimation method is presented in Algorithm \ref{alg}.

\subsubsection{Heuristic Identification of the Remaining Selection Variables}\label{sec:methods:u_pred}

Our bound requires identifying the remaining selection variables $U^* \coloneq U \setminus \tilde{U}$ from the variables observed in $\mathcal{D}_P$. We propose a simple, calibration-based heuristic for nominating $U^*$.

Suppose that the probability of selection can be expressed as $p(S = 1 \mid U) = g(U\beta)$
for some link function $g$ and coefficient vector $\beta$. Let $C := (X \setminus \tilde{X} \setminus Y, U^*)$ denote the set of possible selection variables $U^*$, where $U^*$ cannot overlap with $Y$ or $\tilde{X}$. We can write the probability of selection equivalently as $p(S = 1 \mid \tilde{U}, C) =g(\tilde{U}\omega + C\gamma)$,
where $\gamma_j = 0$ for all variables $C_j \in X$. The task of determining which variables in $C$ are selection variables thus becomes the simpler task of determining which $\gamma_j$ are non-zero.

Similar to existing calibration-based methods \citep{wu2003optimal,kundu2024framework}, we use the following moment-matching estimating equation over $\mathcal{D}_P$ to empirically solve for $\omega$ and $\gamma$:

\begin{align*}
    \sum_{i: S^{(i)} = 1}
    \frac{m(\tilde{U}^{(i)}, \ C^{(i)})}{g(\tilde{U}^{(i)}\omega  + C^{(i)}\gamma )}  
    = 
    \mathbb{E}_{\mathcal{D}_Q}[m(\tilde U, \ C)]
\end{align*}
where $m$ is the user-specified moment map evaluated at each sample and $\mathbb{E}_{\mathcal{D}_Q}[m(\tilde U, \ C)]$ is the corresponding empirical moment vector under the target distribution $Q$.
Given the observation of $\mu_Q$, $\sigma^2$ from Constraint \ref{ass:mu_sigma}, the choice of $m$ may match on first moments, second moments, or their concatenation.

We construct confidence intervals for $\gamma_j$ using the bootstrapped distribution with percentile parameter $\alpha$. The corresponding $C_j$ whose intervals do not contain zero are selected as the estimated $U^*$. Further details, including an adaptation for searching in high dimensions, are presented in Appendix \ref{appendix:method-extensions:heuristic}.

\subsubsection{Density Estimation}\label{method:density_est}

When the variables are fully categorical, the conditional densities $p(\tilde{X}, Y \mid \tilde{U}, \hat{U}^*, S=1)$ and $p(\tilde{X}, Y \mid \tilde{U}, S=1)$ can be estimated from table counts. For data involving continuous variables, the density functions may be estimated using kernel density estimators or conditional normalizing flow models \citep{flow1, flow2}. Estimation of the propensity $w(\tilde{U}) = p(\tilde{U})\ / \ p(\tilde{U}\mid S=1) = p(S=1)\ / \ p(S=1 \mid \tilde{U})$ is even simpler and can be computed by fitting a classifier \citep{clf1} or table counts if data are discrete. Although our main focus is on categorical data given its omnipresence in medical settings, we discuss continuous density estimation in Appendix \ref{appendix:cont_dens_est}.

\begin{algorithm}
\caption{Upper bound estimation on the generalization performance $R_Q$}\label{alg}
\begin{algorithmic}[1]
\Require prediction model $f_P$; biased dataset $\mathcal{D}_P = \{X^{(i)},Y^{(i)}, U^{(i)}, S^{(i)}=1\}_i$; target dataset $\mathcal{D}_Q = \{\tilde{U}^{(k)}\}_k$; external means and variances  $\mu_Q(V), \sigma^2_Q(V)$ $\forall V \in (X, U)$;  $\alpha$-level for heuristic algorithm
\Ensure $\hat{R}_Q$, as defined in Theorem \ref{thm:bound}

\State $\hat{U}^* \gets$ output selection variables from heuristic search with significance $\alpha$
\State Estimate the conditional density functions $p(\tilde{X}, Y \mid \tilde{U}, \hat{U}^*, S=1)$ and $p(\tilde{X}, Y \mid \tilde{U}, S=1)$ using $\mathcal{D}_P$
\State Estimate the propensity weight $w(\tilde{U}) = {p(\tilde{U})} \ / \ {p(\tilde{U} \mid S=1)}$ using both $\mathcal{D}_P, \mathcal{D}_Q$
\State $w \gets$ empty weight vector of dimension $|\mathcal{D}_P|$
\ForAll{samples $(\tilde{X}^{(i)}, Y^{(i)}, \tilde{U}^{(i)}) \in \mathcal{D}_P$}
    \State $\phi_1 \gets \max\limits_{{u^* \in \hat{\mathcal{U}}}} p(\tilde{X}^{(i)}, Y^{(i)} \mid \tilde{U}^{(i)}, u^*, S=1)$
    \State $\phi_2 \gets p(\tilde{X}^{(i)}, Y^{(i)} \mid \tilde{U}^{(i)}, S=1)$
    \State $w_i \gets w(\tilde{U}^{(i)}) \cdot ({\phi_1} / {\phi_2})$
\EndFor

\State \Return $\hat{R}_Q = \mathbb{E}_{\mathcal{D}_P}\left[w \cdot \ell (f_P(\tilde{X}),Y) \right]$
\end{algorithmic}
\end{algorithm}

\subsubsection{Testing for Assumption Violations When $R_Q$ Is Observed}\label{method:ass_tests}
Our upper bound assumes conditional independence of selection given $U$ (Assumption \ref{assump:cond-indep}) and common support between $P$ and $Q$ (Assumption \ref{ass:support}). However, these assumptions can fail empirically in finite-sample settings. 

If the true $R_Q$ is known, as in settings of simulated selection bias, we can explicitly test how each assumption violation affects the bound error by decomposing our method's bound error $\hat{R}_Q - R_Q$ into a telescoping sum of three factors: 
\begin{align}\label{eq:bound-decomp}
    \hat{R}_Q - R_Q &= \Delta_{\text{TBE}}
    + \Delta_{\text{CI}} 
    + \Delta_{\text{CS}}
\end{align}
where each $\Delta_{(\cdot)}$ term is defined in Appendix \ref{appendix:proofs:decomp}. At a high level, $\Delta_{\text{CI}}$ is zero when the \underline{C}onditional \underline{I}ndependence assumption holds, and $\Delta_{\text{CS}}$ is zero if the \underline{C}ommon \underline{S}upport assumption holds. Finally, $\Delta_{\text{TBE}}$ measures the \underline{T}heoretical \underline{B}ound \underline{E}rror, the gap between our bound estimate $\hat{R}_Q$ and the true $R_Q$ when the two aforementioned assumptions are satisfied. This decomposition therefore quantifies how violations of conditional independence and common support cause the final bound error to deviate from the theoretical bound error.
\subsubsection{Testing for Assumption Violations in Practice}\label{method:practical_tests}
However, the telescoping sum in Equation \ref{eq:bound-decomp} is usually intractable as $R_Q$ is often unknown. Therefore, to approximately test for assumption violations, we present three diagnostics that are straightforward, computationally inexpensive, and can be readily applied using our code implementation. We provide additional details on the diagnostics in Appendix \ref{appendix:ass_viol_practice}. 
\paragraph{Common Support Diagnostics}
\begin{enumerate}
        \item \textbf{KS Test:} The overlap between the observed propensity distribution $p(S=1 \mid \tilde{U})$ under $P$ and $Q$ can be easily evaluated using a Kolmogorov-Smirnov (KS) test, or a similar statistical test.
        \item \textbf{Weight Design Effect $d_{\text{eff}}$:} Moment-matching methods, such as our proposed heuristic in Section \ref{sec:methods:u_pred}, often exhibit instability or poor convergence behavior if the two distributions lack overlap. The stability of the resulting weights, measured via the design effect $d_{\text{eff}}$ \cite{kish1992weighting}, can diagnose potential violations of common support.
    \end{enumerate}
\paragraph{Conditional Independence Diagnostic} 
    \begin{enumerate}[start=3]
        \item \textbf{Propensity Invariance:} We propose a diagnostic that approximately assesses the conditional independence assumption $p(V \mid U,S=1) = p(V \mid U,S=0)$, $\forall V\in (X,Y)$. However, under Constraints \ref{ass:no_samples} - \ref{ass:mu_sigma}, the true $U$ is unknown and we only observe samples where $S=1$. We instead use the observed propensity $\hat{S} = p(S=1 \mid \tilde{U})$ and predicted $\hat{U} := (\tilde{U}, \hat{U}^*)$ from our moment-matching method to assess for equality across $p(V \mid \hat{U},\hat{S}=s_1) = p(V \mid \hat{U},\hat{S}=s_2)$, $\forall s_1,s_2\in [0,1], \forall V\in (X,Y)$. 
    \end{enumerate}

\section{Experimental Setup}\label{sec:experiments}

\subsection{Data}\label{sec:experiments:data}
We evaluate the quality of our bound $\hat{R}_Q$ in three data settings: (i) fully synthetic data; (ii) semi-synthetic data, where we simulate an EHR-specific selection mechanism in clinical data from All of Us; and (iii) real-world selection bias in MIMIC-IV. We provide additional details for each dataset, including preprocessing steps, in Appendix \ref{appendix:data}. 

\subsubsection{Synthetic Data.}
To generate the fully synthetic target dataset $\mathcal{D}_Q$, we sample binary variables $U$ with bivariate correlation $\rho$, and then generate binary $X$ and $Y$ as linear logistic functions of $U$. We then sample our biased dataset $\mathcal{D}_P = \{(X,U,Y,S) \in \mathcal{D}_Q : S = 1\}$ through a logistic selection model:
\begin{align*}
p(S=1 \mid U) = \frac{1}{1+\text{exp}(-g(U\beta))} \\ 
S \sim \text{Bernoulli}(p(S=1 \mid U))
\end{align*} 
where we test both linear and nonlinear link functions, $g$. 

\subsubsection{All of Us.}
The All of Us Research Program \citep{aou} is a demographically diverse biobank based in the U.S. with over 600,000 participants. It includes sociodemographic and biomarker information collected at enrollment, along with longitudinal outcomes from linked medical records.


To form the target dataset $\mathcal{D}_Q$, we filter All of Us participants to a cohort of 255,612 participants. We simulate selection of the biased dataset $\mathcal{D}_P$ given the same logistic selection mechanism described in the fully synthetic setting, where we explicitly define EHR-specific selection variables $U$ (e.g., income level or insurance status) \citep{ehrbias1,ehrbias2,ehrbias3,ehrbias4,biobankbias1,ipw1}. For prediction, we consider 19 health outcomes $Y$ (e.g., Type 2 diabetes mellitus) and 41 binary features $X$ (e.g., blood pressure, lifestyle factors).

\subsubsection{MIMIC-IV.} MIMIC-IV is a deidentified dataset containing over 200,000 patients admitted to the emergency department at the Beth Israel Deaconess Medical Center in the U.S. \citep{johnson2023mimiciv, goldberger2000physionet}. MIMIC-IV is a widely-used benchmark in machine learning, and it is therefore critical to audit whether selection bias leads to performance degradation when models are generalized to broader populations.

We treat MIMIC-IV as the biased dataset $\mathcal{D}_P$ and All of Us as the target dataset $\mathcal{D}_Q$. Because MIMIC-IV contains data from a single U.S. hospital and All of Us provides nationally-representative data, this setup naturally reflects the realistic scenario of a complex and unknown selection mechanism. Furthermore, using All of Us as the target enables method validation against the true $R_Q$ when the variables $\tilde{X},Y$ are observed. 

We conduct three real-world experiments: first, we consider two prediction tasks ($Y$=hypertension and $Y$=Type 2 diabetes mellitus) where $\tilde{X},Y$ variables are observed in both datasets, enabling validation of our bound against the true $R_Q$; second, we evaluate one task ($Y$=hospital mortality) that reflects the realistic scenario where $R_Q$ is unknown. Motivated by prior work in EHR-specific biases \citep{ehrbias1,ehrbias2,ehrbias3,ehrbias4}, we select the selection variables $\tilde{U}$ as a subset of age, insurance type, and primary language. 



\subsection{Prediction Tasks}\label{sec:experiments:sel-pred-mdl}
For the prediction model $f_P$, we learn $p(Y \mid \tilde{X},S=1) = f_P(\tilde{X)}$ either using XGBoost \cite{chen2016xgboost} or an elastic net regularized logistic regression model with class-balancing weights and regularization parameters chosen via cross-validation. In practice, we do not observe a substantial difference in bound characteristics under different prediction models. For each set of experiments, we run a data-driven search to identify $n_{\text{tasks}}$ prediction tasks such that the resulting generalization gap is sufficiently large, i.e., $R_Q - R_P >t_R$. We outline this search in Algorithm \ref{alg:synthetic-selection}.

\subsection{Evaluating Our Proposed Bound in Simulated Selection Settings}\label{sec:experiments:bound_props}
We first evaluate our bound estimation method in the fully synthetic and All of Us data settings. By simulating selection, we can validate if our method, which assumes limited target data observability (Constraints \ref{ass:no_samples} - \ref{ass:mu_sigma}), actually recovers the true expected loss $R_Q$ and selection variables $U$ in practice. We run the following experiments, which are described in more detail in Appendix \ref{appendix:results}:
\subsubsection{Correctness of Heuristic Identification of Selection Variables.} We test how well our moment-matching heuristic (Section \ref{sec:methods:u_pred}) recovers the remaining selection variables compared to random selection and selection based on maximum correlation (regular and Cohen's $d$) with $\tilde{U}$. For each selection strategy, we compute the F2 score, precision, recall, and the Jaccard index of the selected $\hat{U}^*$ compared to the true $U^*$.

\subsubsection{Robustness to Assumption Violations.} We examine how violations of common support and conditional independence may affect the behavior of our bound estimate. To control the degree of assumption violation, we vary four parameters: sample size $|\mathcal{D}_Q|$, the strength of the selection mechanism, covariate imbalance, and the number of extraneous features $X \setminus \tilde{X}$. For each parameter setting, we decompose the estimated bound error $\hat{R}_Q - R_Q$ into the telescoping sum from Equation \ref{eq:bound-decomp}. 

\subsubsection{Validating Our Bound Estimate.} For each prediction task, we run our method on all possible observed subsets $\tilde{U} \subseteq U$ and compute the estimated upper bound $\hat{R}_Q$: first, using the true selection variables, denoted as `UB (true $U^*$)'; second, using our heuristic, denoted as `UB (heuristic $U^*$)'. We then compare the estimated generalization gap (or bound error) with the true quantity. 

We compare against the following baselines: \textit{naive inverse probability of participation weighting} (IPPW) \citep{ipw2000, ipw1,ipw2,ipw3,ipw4}, which estimates sample weights using the fully observed $\tilde{U}$; \textit{empirical calibration} \citep{surveyest}, which estimates sample weights to balance the first moments of all variables in $P$ and $Q$; \textit{entropy balancing} \citep{eb}, a form of calibration that additionally matches second moments; and \textit{raking} (iterative proportional fitting) \citep{surveyest}, a form of calibration that aligns categorical sample data to target table counts. We describe these baselines in detail in Appendix \ref{appendix:baselines}. In Appendix \ref{appendix:other_baselines}, we also provide results on KLIEP \cite{kliep}, KMM \cite{kmmog,kmm}, logistic regression classification \cite{da2,clf1}, RuLSIF \cite{liu2013change}, and uLSIF \cite{kanamori2009least}, which assume unrealistic data availability and are excluded from the main analysis. 

\subsection{Application of Proposed Bound to Real-World Settings}
We next validate our method in real-world settings and discuss how to practically use our method for model auditing.

\subsubsection{Robustness to Assumption Violations.} Assuming $R_Q$ is unknown, we apply the three proposed assumption violation diagnostics from Section \ref{method:practical_tests} on fully synthetic data and outline how to interpret the results in practice.

\subsubsection{Validating Our Bound Estimate.} We evaluate our method on three tasks with real selection bias in MIMIC-IV when compared to the more diverse target population in All of Us. In two tasks, we compare our method's predicted $\hat{R}_Q$ to the true $R_Q$. For the third task, we estimate $\hat{R}_Q$ and provide guidelines for practically validating our method when $R_Q$ is unknown.

\begin{figure}[ht!]
    \centering
\includegraphics[width=\linewidth]{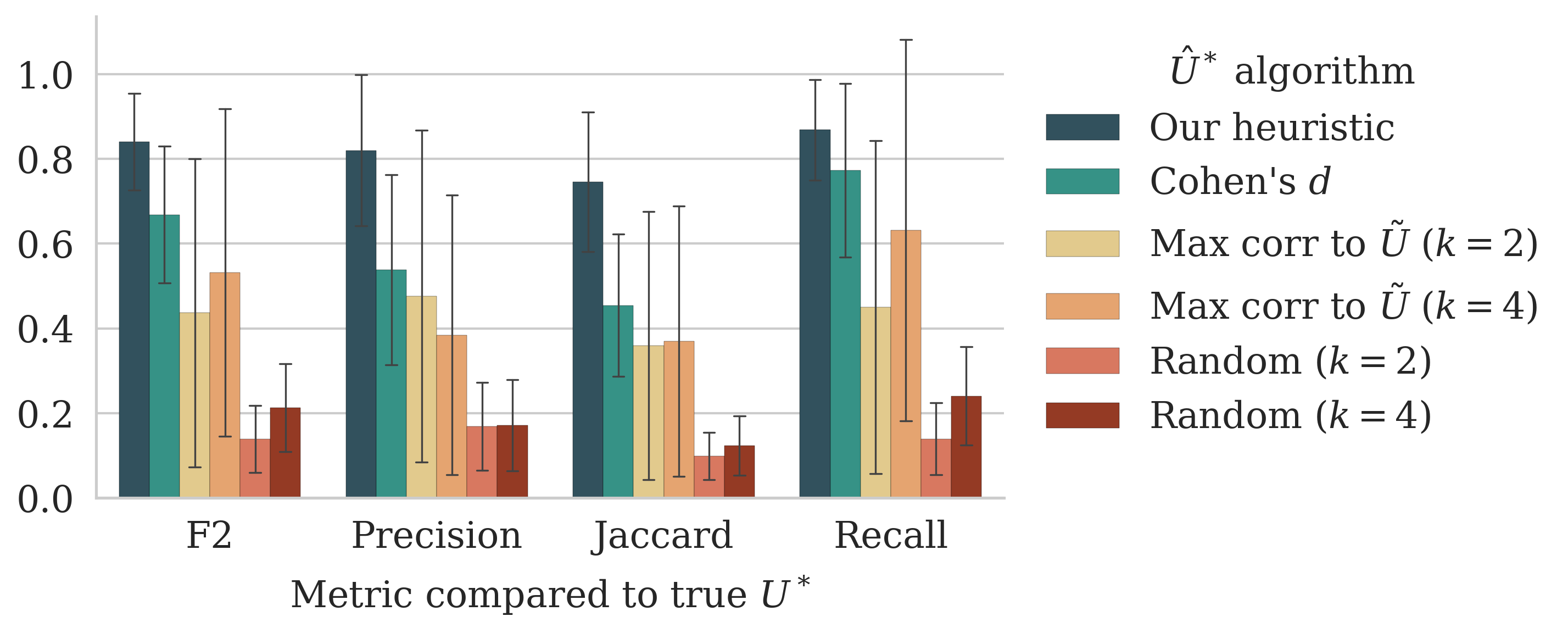}
    \caption{Comparing our heuristic identification algorithm against other baselines in identifying the true selection variables $U^*$ in fully synthetic data. Our heuristic yields high accuracy and thus is a close approximation to the bound under true selection variable identification.}
    \label{fig:synth_heuristic}
\end{figure}

\begin{figure}[ht!]
    \centering
        \includegraphics[width=\linewidth]{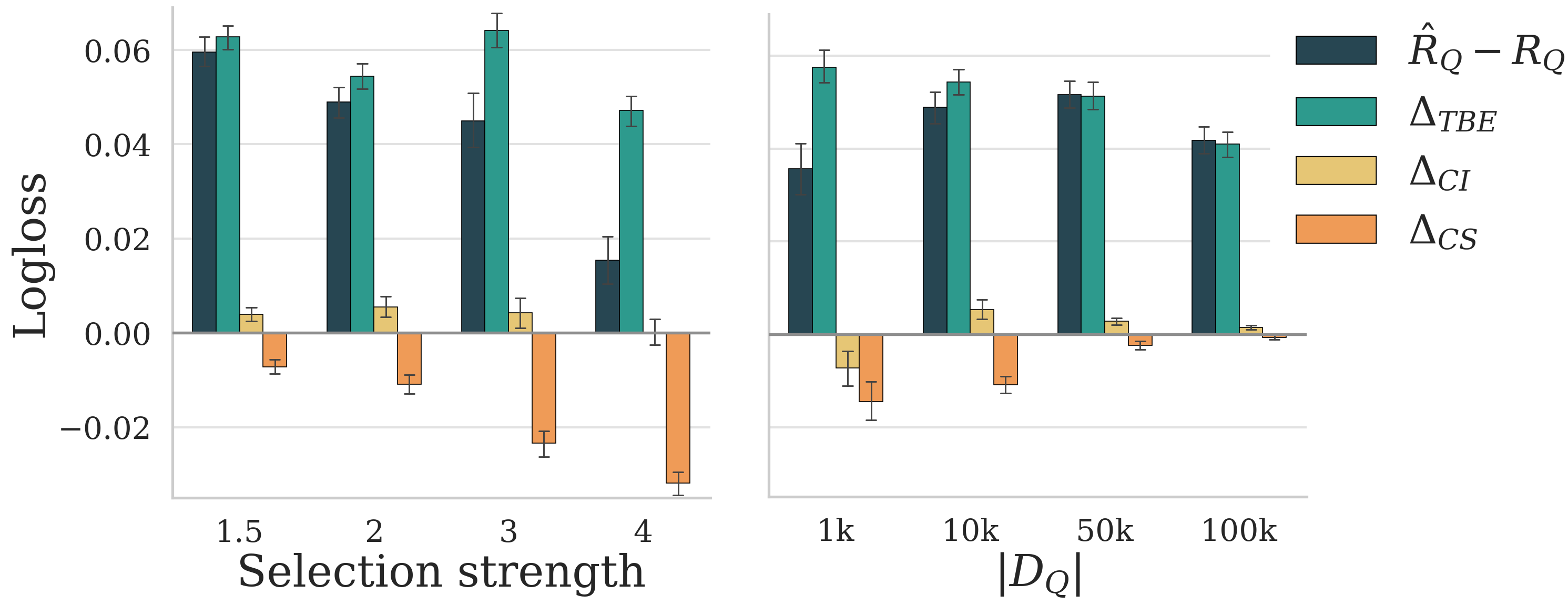}
    \caption{Decomposition of the bound error $\hat{R}_Q - R_Q$ into three terms (see Section \ref{method:ass_tests}) to evaluate assumption violations in fully synthetic data, as sample size and selection strength vary. Extreme violations of common support (i.e., large $|\Delta_{\text{CS}}|$ ) or conditional independence (i.e., large $|\Delta_{\text{CI}}|$ ) may lead to an invalid upper bound.}
    \label{fig:decomp}
\end{figure}

\section{Results}\label{sec:res}
\subsection{Evaluating Our Proposed Bound in Simulated Selection Settings}
\subsubsection{Correctness of Heuristic Identification of Selection Variables.} In Figures \ref{fig:synth_heuristic} and \ref{fig:real_heuristic}, we demonstrate that our proposed moment-matching heuristic outperforms all baselines and is reasonably able to identify the remaining selection variables $U^*$ with F2 scores of 0.84 and 0.81 for synthetic data and All of Us, respectively. These results suggest that the bound estimated using our heuristic $\hat{U}^*$ closely approximates the bound under the true selection variables $U^*$.

\subsubsection{Robustness to Assumption Violations.}

In Figure \ref{fig:decomp}, we apply the proposed error bound $\hat{R}_Q - R_Q$ decomposition while varying the strength of the selection mechanism and the sample size $|\mathcal{D}_Q|$. Results are shown for fully synthetic data using a linear selection mechanism and are averaged across $n_{\text{tasks}}=20$, all $\tilde{U} \subseteq U$, and 5 random seeds. Additional results on All of Us data, as well as varying covariate $\tilde{X}$ imbalance or increasing $\text{dim}(X \setminus \tilde{X})$, are provided in Appendix \ref{appendix:assmp_viol_oracle}. As expected, aggressive selection, small sample size, and high variable imbalance can violate assumptions of conditional independence or common support, as reflected by larger average $\Delta_{\text{CI}}$ and $\Delta_{\text{CS}}$ terms. In these scenarios, caution using our method -- and perhaps deployment of the model in question -- is warranted until more data can be collected. For instance, in tasks where sample size is insufficient, a large negative $\Delta_{\text{CS}}$ term may dominate due to lack of common support, and $\hat{R}_Q$ may be underestimated. Nonetheless, we find on average our method yields a valid upper bound under reasonable sample sizes and modest selection strength.

\subsubsection{Validating Our Bound Estimate.}
In Table \ref{tab:bound-err-aou}, we demonstrate the quality of our bound estimate across $n_{\text{tasks}}=30$, all subsets $\tilde{U} \subseteq U$, and 20 random seeds for both fully synthetic and All of Us data with a nonlinear selection mechanism. Our results confirm our estimate's empirical validity, with 97\% of all tasks in All of Us yielding a valid upper bound versus 82\% for the next best baseline. Although our bound in theory could be prohibitively large, in practice it is non-vacuous. For instance, the 95th percentile of the bound error in the All of Us experiments is 0.17 logloss. In Appendix \ref{appendix:results}, we present additional results, including performance on linear selection, synthetic continuous, and synthetic high-dimensional data.

We next examine our method's robustness to which selection variables $\tilde{U}$ are fully observed (Constraint \ref{ass:utilde}). In Figure \ref{fig:usens_combined}, we plot the bound error $\hat{R}_Q - R_Q$ based on the dimension $\text{dim}(\tilde{U})$ of observed variables out of $\text{dim}(U)=5$ variables total, using All of Us data and nonlinear selection. As expected, our bound estimate slightly improves as the availability of selection variables $\tilde{U}$ increases. In Figures \ref{appendix:fig:ures_real} and \ref{appendix:fig:ures_synth}, we also show that the bound error is correlated with how well $\tilde{U}$ predicts $S$. However, our method provides reasonable bounds even when $\text{dim}(\tilde{U})=1$, highlighting robustness to settings where target observability is highly limited. Additional results are included in Appendix \ref{appendix:results}. 

In Figure \ref{fig:perf_tasks}, we examine the generalization gap for two tasks in the All of Us data, where we use a linear selection mechanism designed to simulate EHR-specific selection bias. Five additional tasks are plotted in Figure \ref{appendix:fig:udim-real}. We observe that the baselines severely underestimate generalization performance, risking confident deployment of a model that might perform poorly in practice. On the other hand, our method provides a tight and valid upper bound on the real generalization gap.

\begin{table}[ht!]
\centering
\caption{Bound error $\hat{R}_Q - R_Q$ for both fully synthetic and All of Us data. "Validity" denotes the fraction of tasks
that yields a valid upper bound $> \epsilon$, for some small negative $\epsilon$. "$(0.05,\,0.95)$" denotes the 5th and 95th percentiles. "$d_\text{eff}$" is the design effect. Our generalization estimate UB (heuristic $U^*$) provides an non-vacuous upper bound with higher rates of validity than the baselines.}
\label{tab:bound-err-aou}

\footnotesize
\setlength{\tabcolsep}{4pt}
\renewcommand{\arraystretch}{0.95}

\begin{adjustbox}{max width=\columnwidth}
\begin{tabular}{@{}llrccc@{}}
\toprule
& & \multicolumn{4}{c}{$\ \ \ \ \ \ \ \  \ \hat{R}_Q - R_Q$} \\
\cmidrule(lr){3-6}
\textbf{} & \textbf{} & \textbf{$\ \ \ \boldsymbol{\mu} \pm \boldsymbol{\sigma} \ \ \ $} & \textbf{Validity} & \textbf{(0.05,\,0.95)} & \textbf{$\boldsymbol{d}_{\text{eff}}$} \\
\midrule
Synthetic & \textbf{UB (true $U^*$)} & $0.05 \pm 0.06$ & $0.90$ & $(-0.02, 0.15)$ & $4.9$ \\
&\textbf{UB (heuristic $U^*$)}  & $0.06 \pm 0.05$  & $0.90$ & $(-0.02, 0.15)$ & $4.7$ \\
&Naive IPPW of $\tilde{U}$& $-0.03 \pm 0.04$ & $0.29$ & $(-0.11, 0.02)$ & $2.1$ \\
&Raking & $-0.01 \pm 0.02$ & $0.50$ & $(-0.05, 0.02)$ & $3.1$ \\
&Calibration   & $-0.01 \pm 0.02$    & $0.50$ & $(-0.05, 0.02)$ & $3.1$ \\
&Ent. Balancing& $-0.01 \pm 0.02$& $0.50$ & $(-0.05, 0.02)$ & $3.1$ \\
\midrule
All of Us & \textbf{UB (true $U^*$)} & $0.04 \pm 0.04$ & $0.97$ & $(-0.01, 0.12)$ & $28.9$ \\
&\textbf{UB (heuristic $U^*$)} & $0.06 \pm 0.05$ & $0.97$ & $(-0.01, 0.17)$ & $26.3$ \\
&Naive IPPW of $\tilde{U}$& $-0.01 \pm 0.01$  & $0.82$ & $(-0.02, 0.01)$ & $11.8$ \\
&Raking & $-0.01 \pm 0.02$      & $0.79$ & $(-0.06, 0.00)$ & $4.6$ \\
&Calibration     & $-0.01 \pm 0.02$      & $0.79$ & $(-0.04, 0.00)$ & $3.8$ \\
&Ent. Balancing        & $-0.01 \pm 0.01$      & $0.74$ & $(-0.04, 0.00)$ & $5.3$ \\
\bottomrule
\end{tabular}
\end{adjustbox}
\end{table}

\begin{figure}[ht!]
    \centering
\includegraphics[width=\linewidth]{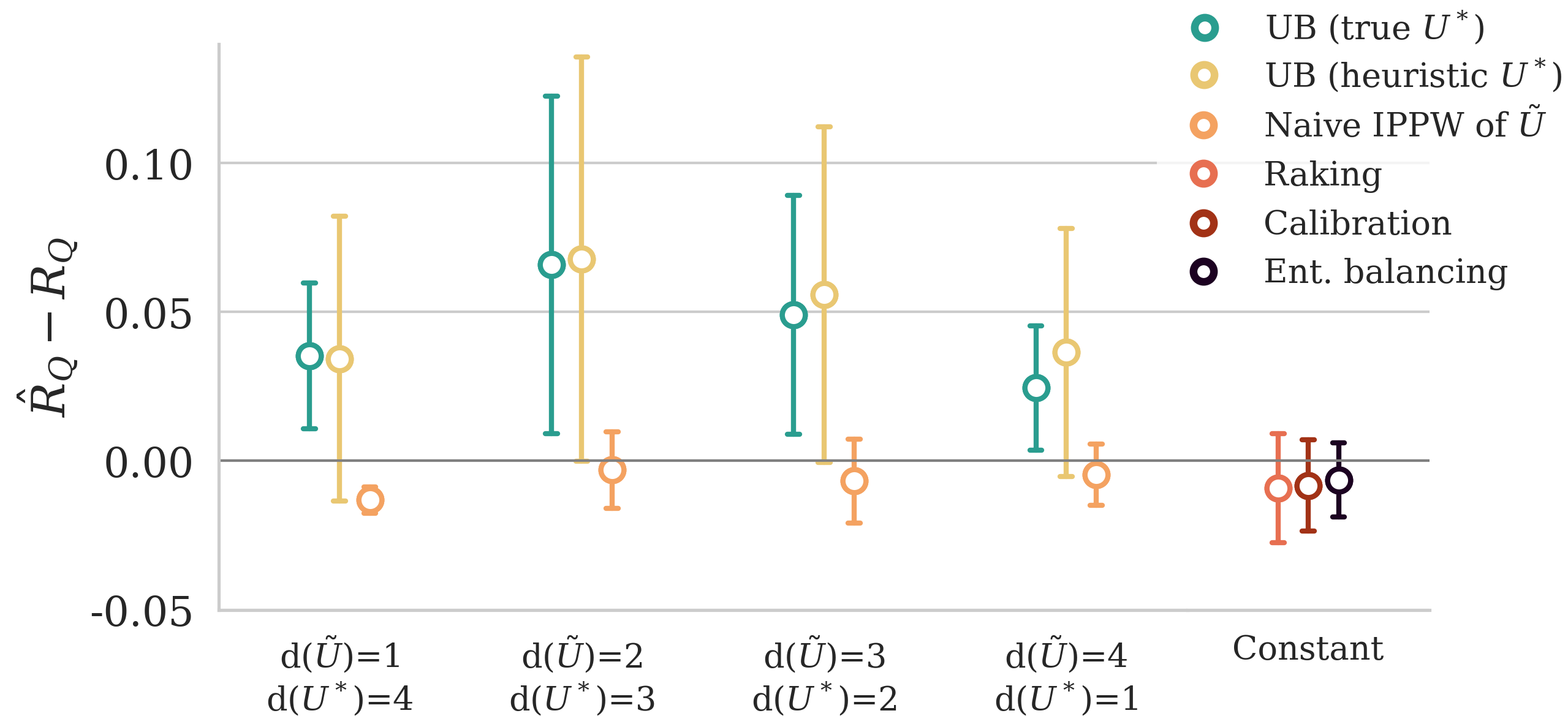}
    \caption{Bound error $\hat{R}_Q-R_Q$ based on the dimension $\text{d}(\tilde{U}):=\text{dim}(\tilde{U})$ of observed selection variables in the All of Us data. While the baselines sometimes underestimate, our method consistently yields a valid upper bound on $R_Q$.}
    \label{fig:usens_combined}
\end{figure}

\subsection{Application of Proposed Bound to Real-World Settings}
\subsubsection{Robustness to Assumption Violations}\label{results:practical_viol_teests}
In Table \ref{tab:violation_tests}, we show the application of our three assumption violation diagnostics on fully synthetic data across $n_{\text{tasks}}=10$, subsets $\tilde{U}\subseteq U$, and 2 seeds; implementation details are provided in Appendix \ref{appendix:ass_viol_practice}. As expected, higher levels of assumption violations (as indicated by higher scores for all diagnostics) occur under low sample size and high selection strength. However, even when assumptions were violated, our bound remained largely valid.

\begin{table}[ht!]
\centering
\caption{Approximate assumption violation diagnostics on fully synthetic data. `CS' and `CI' denote a diagnostic for Common Support and Conditional Independence, respectively. For all diagnostics, a higher score indicates increased violation.}
\label{tab:violation_tests}

\footnotesize
\setlength{\tabcolsep}{4pt}
\renewcommand{\arraystretch}{1}

\begin{adjustbox}{max width=\columnwidth}
\begin{tabular}{@{}cccccc@{}}
\toprule
$|\boldsymbol{\mathcal{D}_Q}|$ &
\makecell{\textbf{Selection} \\ \textbf{strength}} & 
$\makecell{\boldsymbol{\hat{R}_Q} - \boldsymbol{R_Q}\\ \boldsymbol{\mu} \pm \boldsymbol{\sigma}}$ &
\makecell{(\textbf{CS})\\ \textbf{KS Test}} &
\makecell{(\textbf{CS})\\ \textbf{Weight} $\boldsymbol{d_{\text{eff}}}$} &
\makecell{(\textbf{CI})\\ \textbf{Propensity Invariance}} \\
\midrule
1000   & Low  & $0.08 \pm 0.06$ & $0.18 \pm  0.08$ & $50.0 \pm 62.1$ & $0.02 \pm 0.01$ \\
       & High  & $0.07 \pm 0.11 $&  $0.19 \pm  0.10$ & $68.6 \pm 89.0$ & $0.03 \pm 0.01$ \\
100000 & Low   & $0.06 \pm 0.06$ & $0.14 \pm 0.05 $& $1.4 \pm  0.46$ & $0.01 \pm 0.00$ \\
       & High  & $0.03 \pm 0.06$ & $ 0.15\pm 0.06 $& $3.8 \pm  2.85$ & $0.02 \pm 0.00$ \\
\bottomrule
\end{tabular}
\end{adjustbox}
\end{table}

Interpretation of the \textbf{KS Test} statistic and p-value should be relative to the user's tolerance for generalization risk. In addition, because the diagnostic relies on the propensity $p(S=1\mid \tilde{U})$, confidence in the diagnostic can be calibrated to the user's confidence in the observed $\tilde{U}$ capturing the drivers of selection. To interpret the \textbf{Weight $d_{\text{eff}}$} diagnostic, weights may be deemed unstable when $d_{\text{eff}} \gg 1/\rho,$ where $\rho \in [0,1]$ reflects the user’s risk tolerance. For instance, a more conservative user may decide $\rho=0.1$, i.e., common support is approximately satisfied if $d_{\text{eff}} \ll 10$. Finally, interpreting the \textbf{Propensity Invariance} diagnostic similarly depends on the user's confidence in $\tilde{U}$ and the predicted $\hat{U}^*$. Note that while the common support diagnostics do not depend on proper heuristic identification of $U^*$, the Propensity Invariance test will fail, as intended, under imperfect identification of the set of selection variables $\hat{U}:=(\tilde{U}, \hat{U}^*)$. For instance, given one is confident in $\tilde{U}$ (i.e., has determined $\tilde{U}$ are reasonable drivers of selection using domain knowledge), then the failure of the diagnostic indicates an unreliable bound estimate, potentially due to low sample size or inaccurate heuristic identification of $U^*$. Furthermore, because the diagnostic calculates conditional independence for each variable $V \in (X,Y)$ before averaging across variables, users can examine and potentially reconsider using covariates with high individual diagnostic scores.

\subsubsection{Validating Our Bound Estimate.}
In Figure \ref{fig:mimic}, we show the predicted generalization gap of real-world selection bias in MIMIC-IV relative to the broader All of Us population. In the two leftmost tasks where $R_Q$ is known, we observe that despite the complexity of real-world distribution shifts, our method closely matches the true generalization performance and substantially outperforms baseline approaches. 

However, in most real-world scenarios, no such target reference exists. For instance, in the MIMIC-IV experiment of predicting $Y = \text{Mortality}$, the true $R_Q$ and variables $U$ are unknown given limited target distribution availability (Constraints \ref{ass:no_samples} - \ref{ass:mu_sigma}). To validate our proposed generalization bound $\hat{R}_Q$ in practice, users may incorporate the assumption diagnostics, as discussed above in Section \ref{results:practical_viol_teests}. In addition, external domain knowledge -- such as prior literature, known causal structures, or empirical studies on similar generalization settings -- can indicate if a generalization gap is expected. In our setting, for example, prior studies have shown that single-site hospital data often generalize poorly for mortality estimates \cite{singh2022generalizability,hydoub2023risk}, and Berkson’s bias \cite{goldstein2016controlling,berkson1946limitations} in causal inference implies that hospital-based datasets (e.g., MIMIC-IV) tend to oversample sicker patients relative to the general population. These findings suggest that mortality prediction using MIMIC-IV data may be non-generalizable, thus lending credulity to the nonzero gap predicted by our method.

\begin{figure}[h!]
\centering
\includegraphics[width=\columnwidth]{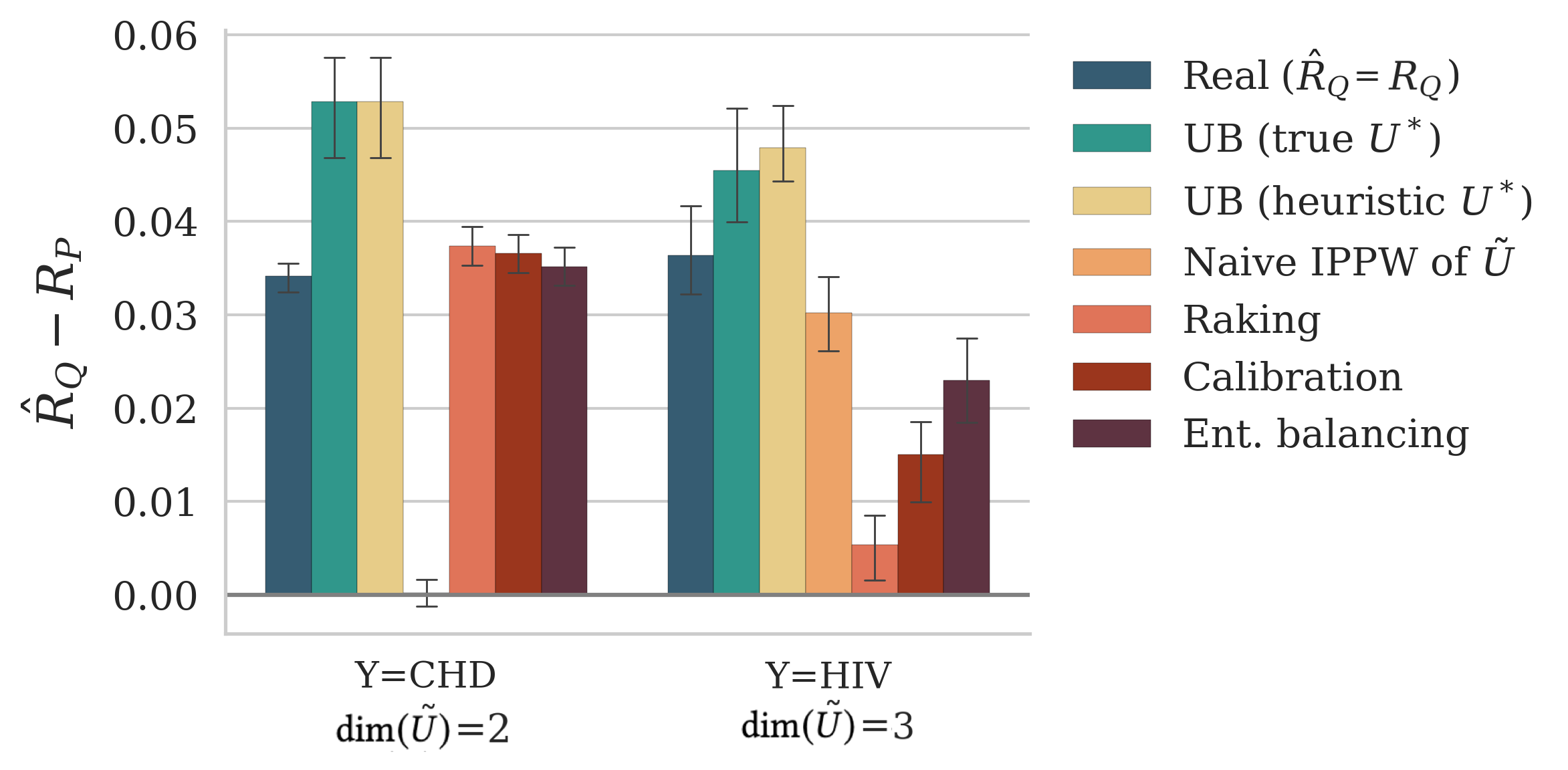}

\vspace{4pt} 

\footnotesize
\setlength{\tabcolsep}{4pt}      
\renewcommand{\arraystretch}{1} 
\begin{adjustbox}{max width=\columnwidth}
\begin{tabular}{@{}l
    >{\raggedright\arraybackslash}p{0.44\columnwidth}
    >{\raggedright\arraybackslash}p{0.44\columnwidth}@{}}
\toprule
& \textbf{Y = Coronary heart disease (CHD)} & \textbf{Y = Human immunodeficiency virus (HIV)} \\
\midrule
\textbf{Covariates $\tilde{X}$} & Age, BMI, vitals, lifestyle features & 
Age, BMI, vitals, lifestyle features \\
\textbf{Known $\tilde{U}$}          & Income (+), confidence in medical system (+) & Age (+), has housing (+), education level (+)          \\
\textbf{Unknown $U^*$}          &   Employment status (+), general health (-)  &    Employment status (+)       \\
\bottomrule
\end{tabular}
\end{adjustbox}

\caption{Generalization gap $\hat{R}_Q-R_P$ (top) and the corresponding prediction task details (bottom) for two specific tasks in the All of Us data. (+) and (-) denotes over- and under-sampling for that selection variable, respectively, guided by EHR-specific selection bias. While baselines risk underestimating generalization performance, our method produces a valid upper bound.}
\label{fig:perf_tasks}
\end{figure}

\begin{figure}[h!]
\centering
\includegraphics[width=\columnwidth]{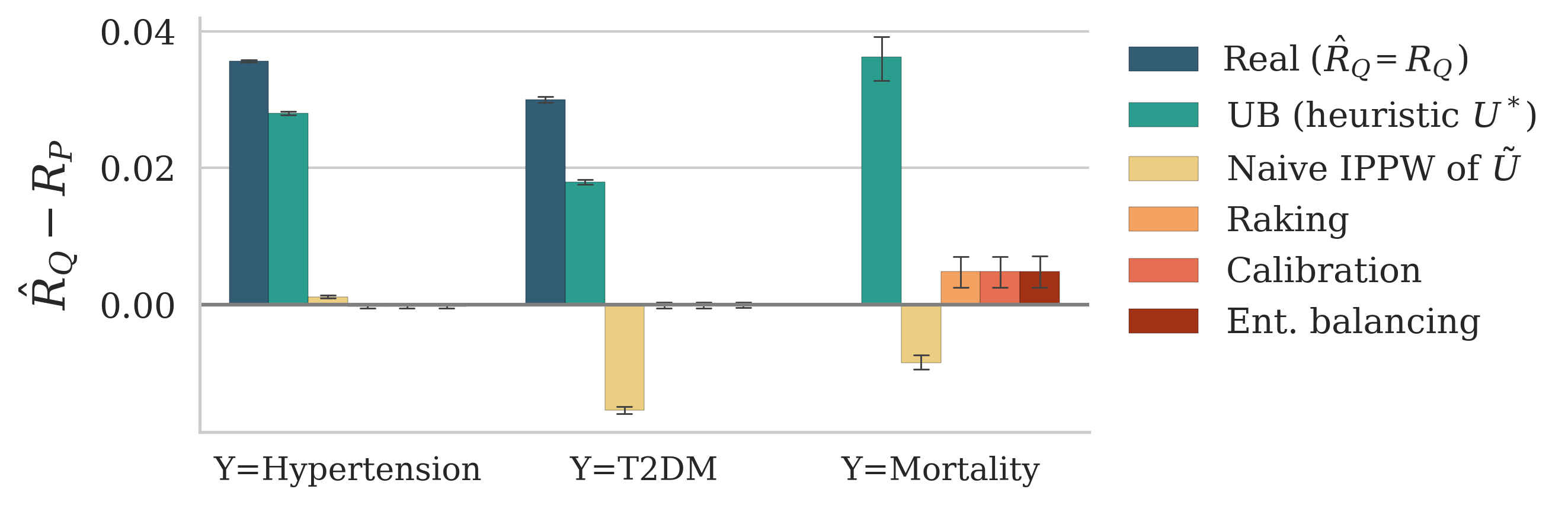}

\vspace{4pt} 
\scriptsize
\setlength{\tabcolsep}{4pt}      
\renewcommand{\arraystretch}{1.15}
\begin{tabularx}{\columnwidth}{@{}l X X X@{}}
\toprule
& \textbf{Y = Hypertension} & \textbf{Y = Type 2 diabetes mellitus (T2DM)} & \textbf{Y = Mortality$^*$} \\
\midrule
\textbf{Covariates $\tilde{X}$} & Diastolic blood pressure, substance abuse disorder, T2DM, Alzheimer's disease, gender & Marital status, race, primary language, anxiety disorder & Weight, T2DM, endometriosis, hypertension, platelet counts$^*$, insurance type, bicarbonate levels$^*$, systolic blood pressure \\
\textbf{Proposed $\tilde{U}$} & Age, English as primary language & Age, English as primary language, insurance type & Age, insurance type  \\
\textbf{Predicted $U^*$}  & Comorbidities, diastolic and systolic blood pressure, insurance type, race, marital status&  Comorbidities, age, diastolic and systolic blood pressure, insurance type, race, marital status  & Comorbidities, age, diastolic blood pressure, insurance type, race, marital status     \\
\bottomrule
\end{tabularx}

\caption{Real-world generalization gap $\hat{R}_Q - R_P$ (top) and the corresponding prediction task details (bottom) for three tasks using the biased MIMIC-IV data, relative to the target All of Us data. `$^*$'indicates that variable is not observed in All of Us. Note for the task Y=Mortality, the real gap ${R}_Q - R_P$ is unknown and thus not plotted.}
\label{fig:mimic}
\end{figure}

Together with the above considerations, our proposed bound can help guide deployment decisions. When the assumption violation diagnostics indicate the underlying assumptions are satisfied and the estimated bound $\hat{R}_Q$ is small, practitioners may proceed with model deployment with greater confidence. Conversely, when the bound is large and prior evidence suggests generalization risk, deployment should be delayed in favor of remediation strategies, such as collecting more representative data or applying model-based debiasing techniques.

\section{Conclusion}
In this work, we propose a practical method for estimating the upper bound on the worst-case performance of prediction models under selection bias. Our work has several limitations. First, we focus on low-dimensional categorical data given its prevalence in medical settings and straightforward density estimation. Extensions to high-dimensional, time-series, and mixed-type data -- as we explore in Appendix \ref{appendix:synth-add-res} -- warrant further research. Second, selection bias is closely tied to historical causes of discrimination and underrepresentation. Future work into model generalization should examine subgroup-specific bounds or, even better, root-cause remedies such as deliberate data collection of historically underrepresented populations \cite{bibbins2022improving}. Third, we recommend exploring strategies to reduce the high variance of our bound, such as density-free approaches. Fourth, while we limited our study to relatively simple prediction models $f_P$, it is also important to understand how selection bias affects more complex models, such as LLMs \cite{guo2024bias}. Finally, general-purpose guidelines for which type of prediction tasks and data structures are most affected by selection bias, as has been done in causal inference \cite{whatif,pearl,t1t2,infante,smith}, would be highly valuable.

Our work makes several important contributions. Unlike existing methods, our method operates under the realistic assumption of limited target data availability. Under these pragmatic data constraints, we propose a novel upper bound on worst-case generalization error and an estimation method using a moment-matching heuristic. We demonstrate through extensive experiments on both synthetic and real-world medical datasets that our method recovers tight and valid bounds even under modest assumption violations. Our experiments on MIMIC-IV highlight the ability of our method to identify real-world cases of selection bias that could lead to generalizability harms if left unaddressed. To this end, we also provide diagnostics for assessing assumption violations, guidance on interpreting the bound, and user-friendly code for running our method. While our work focuses on applications to medical settings, particularly EHR-specific selection bias, the proposed method is applicable to any setting of selection bias. For instance, in Appendix \ref{appendix:cses} we apply our bound to detect potential selection bias in political survey data. We hope our work will serve as a practical tool for researchers and practitioners to audit for selection bias, enabling more trustworthy and generalizable prediction models.

\begin{acks}
We thank Vasilis Syrgkanis, Panagiotis Stanitsas, Dominik Rothenhäusler, Roshni Sahoo, and Sanmi Koyejo for their comments and insights.
\end{acks}
\bibliographystyle{ACM-Reference-Format}
\balance
\bibliography{bib}

\clearpage

\appendix
\onecolumn

\section{Proofs}\label{appendix:proofs}

\subsection{Upper bound}\label{appendix:proofs:ub}
\begin{proof}
    Under Assumption \ref{ass:support} of common support, and by the invariance of the loss term to $\tilde{U}$, we can rewrite $R_Q$ as
    \begin{align*}
        R_Q  &= \mathbb{E}_{P} 
        [\frac{p(\tilde{X}, Y)}{p(\tilde{X}, Y \mid S=1)} \cdot \ell (f_P(\tilde{X}), Y)] \\
        &= \mathbb{E}_{P} 
        [\frac{p(\tilde{X}, Y, \tilde{U})}{p(\tilde{X}, Y, \tilde{U} \mid S=1)} \cdot \ell (f_P(\tilde{X}), Y)] \\
        &= \mathbb{E}_{P} 
        [w(\tilde{U}) \cdot \frac{p(\tilde{X}, Y \mid\tilde{U})}{p(\tilde{X}, Y \mid \tilde{U},  S=1)} \cdot \ell (f_P(\tilde{X}), Y)]
    \end{align*}
    where $w(\tilde{U}) \coloneq\frac{p(\tilde{U})}{p(\tilde{U} \mid S=1)}$. By Assumption \ref{assump:cond-indep} of conditional independence, we know that $p(\tilde{X}, Y \mid U) = p(\tilde{X}, Y \mid U, S=1)$ and thus  
    \begin{align*}
        p(\tilde{X}, Y \mid \tilde{U}) &= \int_{u^* \in \mathcal{U}^*} p(u^* | \tilde{U})p(\tilde{X},Y|\tilde{U},u^*)du^* \\ &= \mathbb{E}_{Q_{U^* \mid \tilde{U}}}[p(\tilde{X}, Y \mid \tilde{U}, U^*)] \\
        &= \mathbb{E}_{Q_{U^* \mid \tilde{U}}}[p(\tilde{X}, Y \mid \tilde{U}, U^*, S=1)] 
    \end{align*}
    If we substitute this into the above equation for $R_Q$, we have
        \begin{align*}
        R_Q = \mathbb{E}_{P}
        \left[
        \frac{\mathbb{E}_{Q_{U^* \mid \tilde{U}}}[p(\tilde{X}, Y \mid \tilde{U}, U^*, S=1)]}{p(\tilde{X}, Y \mid \tilde{U}, S=1)} \cdot w(\tilde{U}) \cdot \ell (f_P(\tilde{X}), Y)\right]
    \end{align*}
    By definition of expectation and non-negativity of density functions, there exists some $\varepsilon_{\tilde{X}, Y, \tilde{U}} \geq 0$ such that
    \begin{align*}
        \max_{u^* \in \mathcal{U}^*}p(\tilde{X}, Y \mid \tilde{U}, u^*, S=1) = \mathbb{E}_{Q_{U^* \mid \tilde{U}}}[p(\tilde{X}, Y \mid \tilde{U}, u^*, S=1)] + \varepsilon_{\tilde{X}, Y, \tilde{U}}
    \end{align*}
    Thus
    \begin{align*}
        R_Q &= \mathbb{E}_{P} 
        \left[
        \frac{\max_{u^* \in \mathcal{U}^*}p(\tilde{X}, Y \mid \tilde{U}, u^*, S=1) - \varepsilon_{\tilde{X}, Y, \tilde{U}}}{p(\tilde{X}, Y \mid \tilde{U}, S=1)} \cdot w(\tilde{U}) \cdot \ell (f_P(\tilde{X}), Y)\right] \\
        &= \mathbb{E}_{P} 
        \left[
        \frac{\max_{u^* \in \mathcal{U}^*}p(\tilde{X}, Y \mid \tilde{U}, u^*, S=1)}{p(\tilde{X}, Y \mid \tilde{U}, S=1)} \cdot w(\tilde{U}) \cdot \ell (f_P(\tilde{X}), Y)\right] \\
        &\quad\quad\quad\quad- \mathbb{E}_{P} 
        \left[
        \frac{\varepsilon_{\tilde{X}, Y, \tilde{U}}}{p(\tilde{X}, Y \mid \tilde{U}, S=1)} \cdot w(\tilde{U}) \cdot \ell (f_P(\tilde{X}), Y)\right] \\
        &\leq \mathbb{E}_{P} 
        \left[
        \frac{\max_{u^* \in \mathcal{U}^*}p(\tilde{X}, Y \mid \tilde{U}, u^*, S=1)}{p(\tilde{X}, Y \mid \tilde{U}, S=1)} \cdot w(\tilde{U}) \cdot \ell (f_P(\tilde{X}), Y)\right] \\
    \end{align*}
    since $p(\tilde{X}, Y \mid \tilde{U}, S=1) \geq 0$ and $\ell (f_P(\tilde{X}), Y) \geq 0$ from Assumption \ref{assump:pos-loss}. 
\end{proof}

\subsection{Bound decomposition}\label{appendix:proofs:decomp}

In Section \ref{sec:experiments:bound_props}, we proposed that the bound error could be decomposed as follows:
\begin{align*}
    \hat{R}_Q - R_Q &= \Delta_{\text{TBE}}
    + \Delta_{\text{CI}} 
    + \Delta_{\text{CS}}
\end{align*}
 where $\Delta_{\text{CS}}$ is the error contribution from violating the necessary assumption of \underline{C}ommon \underline{S}upport between $P$ and $Q$; and $\Delta_{\text{CI}}$ is the error contribution from violating the necessary assumption of \underline{C}onditional \underline{I}ndependence given $U$. In theory (i.e., under infinite samples), these two assumptions hold and $\Delta_{\text{CI}} = \Delta_{\text{cs}} = 0$, meaning $\hat{R}_Q - R_Q = \Delta_{\text{TBE}}$, where $\Delta_{\text{TBE}}$ is the \underline{T}heoretical \underline{B}ound \underline{E}rror when all assumptions hold. 
 
Here, we formally define each of the factors $\Delta_{\text{TBE}}$, $\Delta_{\text{CI}}$, and $\Delta_{\text{CS}}$. Notice that $\hat{R}_Q - R_Q$ can be written as the following telescoping sum:
\begin{align*}
    &\hat{R}_Q - R_Q \notag \\
    &= \mathbb{E}_P \left[
        \phi (\tilde{X}, Y, \mathcal{U}^*) \cdot w(\tilde{U}) \cdot \ell(f_P(\tilde{X}), Y) \right] \notag \\
    & - \mathbb{E}_P \left[ 
        \frac{\mathbb{E}_{Q_{U^* \mid \tilde{U}}} [p(\tilde{X}, Y \mid \tilde{U}, U^*, S = 1)]}
             {p(\tilde{X}, Y \mid \tilde{U}, S = 1)} 
        \cdot w(\tilde{U}) \cdot \ell(f_P(\tilde{X}), Y)  \right] \\
        \tag{$\Delta_{\text{TBE}}$} \\
    & + \mathbb{E}_P \left[ 
        \frac{\mathbb{E}_{Q_{U^* \mid \tilde{U}}}  [p(\tilde{X}, Y \mid \tilde{U}, U^*, S = 1)]}
             {p(\tilde{X}, Y \mid \tilde{U}, S = 1)} 
        \cdot w(\tilde{U}) \cdot \ell(f_P(\tilde{X}), Y)  \right] \notag \\
    & - \mathbb{E}_P \left[
        \frac{p(\tilde{X}, Y \mid \tilde{U})}
             {p(\tilde{X}, Y \mid \tilde{U}, S = 1)} 
        \cdot  w(\tilde{U}) \cdot \ell(f_P(\tilde{X}), Y)  \right] \\ 
        \tag{$\Delta_{\text{CI}}$} \\
    & + \mathbb{E}_P \left[
        \frac{p(\tilde{X}, Y \mid \tilde{U})}
             {p(\tilde{X}, Y \mid \tilde{U}, S = 1)} 
        \cdot  w(\tilde{U}) \cdot \ell(f_P(\tilde{X}), Y)  \right] \notag \\
    & - \mathbb{E}_Q \left[ \ell(f_P(\tilde{X}), Y) \right]
    \tag{$\Delta_{\text{CS}}$}
\end{align*}

Defining the first difference as $\Delta_{\text{TBE}}$, the second difference as $\Delta_{\text{CI}}$, and the third difference as $\Delta_{\text{CS}}$, we have the proposed decomposition. These definitions clarify how and why $\Delta_{\text{CI}}$ and $\Delta_{\text{CS}}$ reveal assumption violations. First, $\Delta_{\text{CI}}$ tests if conditional independence is violated by isolating the effect of changing ${E}_{Q_{U^* \mid \tilde{U}}} [p(\tilde{X}, Y \mid \tilde{U}, U^*, S = 1)]$ to ${E}_{Q_{U^* \mid \tilde{U}}} [p(\tilde{X}, Y \mid \tilde{U}, U^*)]  = p(\tilde{X}, Y \mid \tilde{U})$, where the two terms should be equal if conditional independence given $U$ holds. Second, $\Delta_{\text{CS}}$ measures violation of the assumption of common support between $P$ and $Q$ by characterizing the difference between the density ratio re-weighted expectation over $P$ and the unweighted expectation over $Q$. These expressions are equal provided $P$ and $Q$ share common support.

\section{Heuristic $U^*$ selection}\label{appendix:method-extensions:heuristic}

\subsection{Details on bootstrapped confidence intervals}
Recall the heuristic solves for 

\begin{align}\label{app:eq:heuristic}
    \sum_{i: S^{(i)} = 1}
    \frac{m(C^{(i)}, \ \tilde{U}^{(i)})}{g(\tilde{U}^{(i)}\omega  + C^{(i)}\gamma )}  
    = 
     \mathbb{E}_{\mathcal{D}_Q}[m(\tilde U, \ C)]
\end{align}

which computes point estimates for the $\gamma_j$ coefficients. To derive confidence intervals for the estimated $\hat{\gamma}_j$, we take a bootstrapping approach. Specifically, we resample the dataset $\mathcal{D}_P$ with replacement $B$ times. We then solve the above equation for each dataset sample, yielding the bootstrapped distribution $\hat{\gamma}^{(1)}_j, \hat{\gamma}^{(2)}_j, \dots, \hat{\gamma}^{(B)}_j$ for each $j$. We construct a 1-$\alpha$ level confidence interval as $(\hat{\gamma}_{{j}_{(\alpha/2)}}, \hat{\gamma}_{{j}_{(1 - \alpha/2)}})$, where $\hat{\gamma}_{{j}_{(\alpha/2)}}$ denotes the $\alpha/2$ percentile of the bootstrapped coefficients and $\hat{\gamma}_{{j}_{(1 - \alpha/2)}}$ denotes the $1 - \alpha/2$ percentile. In our experiments, we set $B = 1000$ iterations and $\alpha = 0.005$. We set $\alpha$ to be smaller than typically used for bootstrapped confidence intervals because the resulting conservative confidence intervals for $\hat{\gamma}_j$ help limit the number of false positives (variables mistakenly nominated as $U^*$).

\subsection{Extensions to higher dimensions}

Our heuristic's runtime might become infeasible under very high dimensions. Thus, we present a modification that, although an approximation, we found works well in practice. Intuitively, if we can choose a subset of $C: =(X, U^*)$ that are the most likely $U^*$ candidates, we reduce the effective search dimension and decrease runtime.

To narrow down the search space, note a sufficient but not necessary condition of selection variables $U$ is nonzero ``effect size" when the selection indicator $S$ is viewed as a treatment and $U$ is viewed as an outcome. Using our summary statistics, we calculate the (linear) effect size for all variables $V \in (X,U^*)$ via Cohen's $d$ statistic with pooled standard deviation: 
$$d(V) = \frac{|\mu_P(V) - \mu_Q(V)|}{\sigma_{\text{pooled}}}$$

where $N = |\mathcal{D}_P|$, $M =  |\mathcal{D}_Q|$, and $$\sigma_{\text{pooled}} :=\sqrt{\frac{(N - 1) \cdot \sigma^2_P(V) + (M - 1) \cdot \sigma^2_Q(V)}{N + M - 2}}$$ 

We then choose as the candidate set $\tilde{C}$ all variables $V$ with a nonzero effect size, i.e., one that is greater than some small $\delta$: 
 $$\tilde{C} = \{V: V \in (X, U^*) \text{ and } d(V) \geq \delta\}$$ In our experiments, we chose $\delta = 0.05$.
\section{Data}\label{appendix:data}
\subsection{Fully synthetic data}\label{appendix:experiments:synth}

\subsubsection{Data-generating process for synthetic data}\label{appendix:dgp_synth}

In the \textbf{binary data setting}, we generated $U$ as (potentially correlated) Bernoulli random variables. Details on how to generate binary variables with a target correlation $\rho$ are discussed below. We then generated $\tilde{X}$ and $Y$ as Bernoulli random variables parameterized by $p_{\tilde{X}}(U)$ and $p_Y(U)$, respectively, 
\begin{align*} 
    p_{\tilde{X}}(U) &:= p(\tilde{X} = 1 \mid U) = \frac{1}{1 + \exp (\phi^{(d_X)}(U){\beta}_{U\rightarrow \tilde{X}} + \beta_{0, \tilde{X}} + \varepsilon_{\tilde{X}})}  \\
    \tilde{X} &\sim \text{Bernoulli}(p_{\tilde{X}}(U)) \\
    p_Y(U) &:= p(Y = 1 \mid \tilde{X},U) = \frac{1}{1 + \exp(\phi^{(d_Y)}(U, \tilde{X}){\beta}_{U\tilde{X} \rightarrow Y} + \beta_{0, Y} + \varepsilon_Y) } \\
    Y &\sim \text{Bernoulli}(p_{Y}(U))
\end{align*}
where $\beta_{(\cdot \rightarrow \cdot)}$ are chosen coefficient vectors and $\beta_{0, \tilde{X}}$, $\beta_{0, Y}$ are chosen intercepts (see further details below), $\varepsilon_{\tilde{X}} \sim \mathcal{N}(0, \sigma_{\tilde{X}}^2)$, $\varepsilon_{Y} \sim \mathcal{N}(0, \sigma_Y^2)$, and $\phi^{(d_{\tilde{X}})}, \phi^{(d_Y)}$ are degree $d_{\tilde{X}}$ and $d_Y$ polynomial expansions, respectively. We partitioned the remaining $X \setminus \tilde{X}$ into two sets, $X_C$ and $X_U$, where $X_C$ were generated as independent Bernoulli variables and $X_U$ were generated in the same manner as $\tilde{X}$, but were not included as inputs to the data-generating process for $Y$. The relationship between the variables $U$, $\tilde{X}$, $X_C$, $X_U$, and $Y$ are summarized in the graphical model in Figure \ref{appendix:fig:dag}.

In the \textbf{continuous data setting}, we generated $U$ as either normal random variables, $U \sim N(0, \mathbf{\Sigma})$, or as uniform random variables $U \sim \text{Unif}(0, 1)$. We then generated $\tilde{X}$ as a linear function of $U$
\begin{align*}
    \tilde{X} &= U{\beta}_{U\rightarrow \tilde{X}} + \beta_{U} + \varepsilon_{{\tilde{X}}}
\end{align*}
We kept $Y$ as a binary variable, generated in the same manner as described above.

\subsubsection{Coefficient sampling and scaling} We sample the selection coefficient vectors $\beta_{U \rightarrow S}$, $\beta_{0, S}$ as described above, and $\beta_{U \rightarrow \tilde{X}}, \beta_{0, {\tilde{X}}}, \beta_{U \rightarrow Y}, \beta_{\tilde{X} \rightarrow Y}$, $\beta_{0, Y}$ are sampled in a similar fashion. Explicitly, first we sample
$\beta_{(\cdot \rightarrow \cdot)} \sim N(\mathbf{0}, \mathbf{I})$ and $\beta_{0, {\tilde{X}}} \sim N(0, 1), \beta_{0, Y} \sim N(0, 1)$. For binary variables, we then compute the resulting logit standard deviation: 
\begin{align*}
\sigma_{\text{logit}(Y)} &:= \text{stddev}(\phi^{(d_Y)}(U, {\tilde{X}}){\beta}_{U\tilde{X} \rightarrow Y} + \beta_{0, Y} + \varepsilon_Y) \\
\sigma_{\text{logit}({\tilde{X}})} &:= \text{stddev}(\phi^{(d_{\tilde{X}})}(U){\beta}_{U\rightarrow \tilde{X}} + \beta_{0, {\tilde{X}}} + \varepsilon_{{\tilde{X}}})
\end{align*}

Then, we scale the weights based on user-specified target standard deviations  $\sigma^*_{\text{logit}({\tilde{X}})}$ and  $\sigma^*_{\text{logit}(Y)}$:
    \begin{align*}
        (\beta_{U \rightarrow \tilde{X}}, \beta_{0, {\tilde{X}}}) 
        & \leftarrow 
        (\frac{\beta_{U \rightarrow \tilde{X}} \cdot \sigma^*_{\text{logit}({\tilde{X}})}}
        {\sigma_{\text{logit}({\tilde{X}})}},
        \frac{\beta_{0, {\tilde{X}}} \cdot \sigma^*_{\text{logit}({\tilde{X}})}}
        {\sigma_{\text{logit}({\tilde{X}})}}
        ) \\
        (\beta_{U\tilde{X} \rightarrow Y}, \beta_{0, Y})
        & \leftarrow
        (\frac{\beta_{U\tilde{X} \rightarrow Y}\cdot \sigma^*_{\text{logit}(Y)}}
        {\sigma_{\text{logit}(Y)}}, 
        \frac{\beta_{0, Y} \cdot \sigma^*_{\text{logit}(Y)}}
        {\sigma_{\text{logit}(Y)}}
        ) 
    \end{align*}

\begin{figure}
    \centering
    \includegraphics[width=0.5\linewidth]{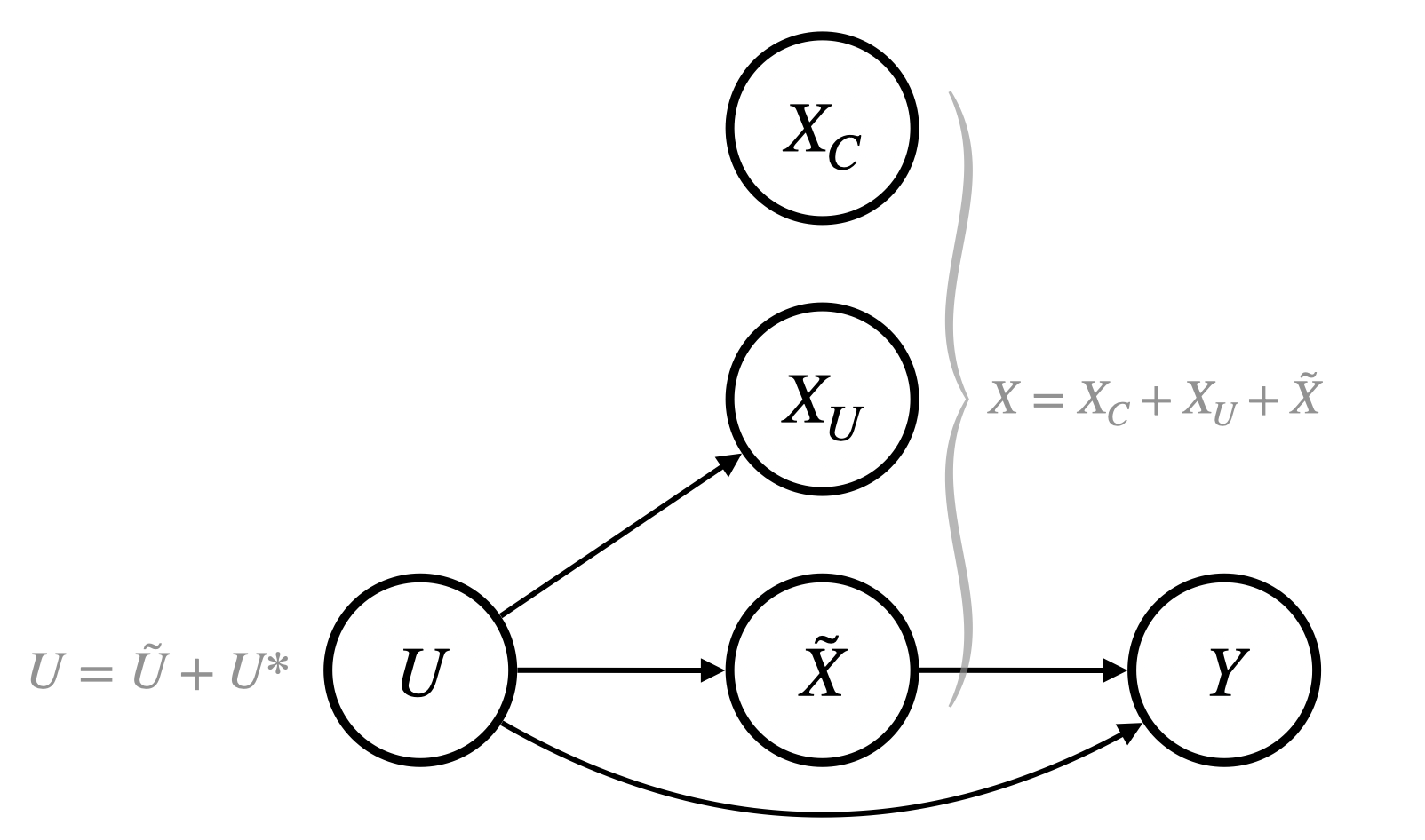}
    \caption{Graphical model depicting relationship between variables $U, X=\tilde{X} +  X_{U} + X$, and $Y$ in our synthetic data experiments.}
    \label{appendix:fig:dag}
\end{figure}

\subsubsection{Generating correlated binary variables, $U$} Our goal is to define a data-generating process for a set of $k$ binary random variables $U_1, U_2, \dots U_k$, such that these random variables satisfy two conditions:
\begin{enumerate}[noitemsep]
    \item For a specified vector of probabilities, $\mathbf{p} = \begin{pmatrix}p_1 & p_2 & \dots & p_k\end{pmatrix}$ it holds that $p(U_j) = p_j \text{ for all } j = 1, \dots k.$
    \item For a specified $k\times k$ correlation matrix $R$, where $R_{ij} \in [-1, 1]$ for all $i \neq j$ and $R_{ij} = 1$ for all $i = j$, $\text{Corr}(U_i, U_j) = R_{ij}$. 
\end{enumerate}
We define a data-generating process that consists of two steps. First, we sample from a multivariate normal distribution with mean zero and covariance matrix $\Sigma$, where $\Sigma$ is a function of the desired correlation matrix $R$ and the vector of probabilities $\mathbf{p}$. Second, we define thresholds based on the desired vector of probabilities $\mathbf{p}$, then apply these thresholds to the sampled normally distributed variables to transform them to binary Bernoulli random variables. 

Setting the thresholds is straightforward. Let $Z_i$ denote an intermediate variable sampled from the zero-mean multivariate normal distribution. We map $Z_i$ to the binary variable $U_i$ using threshold $t_i$, where $U_i = {1}\{Z_i < t_{i}\}$. By setting $t_i = \text{Quantile}(Z_i, p_i)$, it holds that $P(U_i = 1) = p_i$, as desired. 

Sampling $Z_i$ such that the variables $U$ have correlation matrix $R$ is slightly trickier. First, by definition of correlation, we can write $\text{corr}(U_i, U_j)$ as 
\begin{align*}
    \text{corr}(U_i, U_j)
    &= \frac{\mathbb{E}[U_iU_j] - \mathbb{E}[U_i]\mathbb{E}[U_j]}{\sqrt{\text{Var}(U_i) \text{Var}(U_j})} \\
    &= \frac{\mathbb{E}[U_iU_j] - p_ip_j}{\sqrt{p_i(1 - p_i)p_j(1 - p_j)}} \\
    &= \frac{P(U_i = 1, U_j = 1) - p_ip_j}{\sqrt{p_i(1 - p_i)p_j(1 - p_j)}} 
\end{align*}
where $p_i$, $p_j$ are elements of the probability vector $\mathbf{p}$. Using the fact that $U_i, U_j$ are sampled by thresholding $Z_i$, $Z_j$, respectively, we have that $P(U_i = 1, U_j = 1) = P(Z_i < t_i, Z_j< t_j)$. Let $\Sigma_{ij}$ denote the row $i$, column $j$ element of $\Sigma$, so that $\Sigma_{ij} = \text{Cov}(Z_i, Z_j)$. The value of $\Sigma_{ij}$ fully determines the value of $P(Z_i < t_i, Z_j < t_j)$, since $Z_i$ and $Z_j$ are both zero-mean. We want to solve for $\Sigma_{ij}$ such that, by plugging the resulting value of $P(Z_i < t_i, Z_j < t_j)$ into the expression for $\text{corr}(U_i, U_j)$, we obtain the desired value for the correlation, $R_{ij}$. Define the ``error" function $f(\Sigma_{ij}) = \text{corr}(U_i, U_j) - R_{ij}$. The roots of $f$ are then precisely the values of $\Sigma_{ij}$ that produce the desired correlation between $U_i$ and $U_j$. To solve for the roots of $f$, we apply Brent's method \citep{brent} for root-finding over the interval $[-1, 1]$.

\subsubsection{Parameter settings}
In Table \ref{appendix:tab:synth-params-binary}, we show the parameter settings used in our synthetic data experiments for both the binary and continuous data settings.

\begin{table}[ht!]
    \centering
    \resizebox{\textwidth}{!}{
    \begin{tabular}{llll}
    \toprule
    Parameter & Description & \makecell[l]{Value \\ (Binary)} & \makecell[l]{Value \\(Continuous)} \\
    \midrule
    dim($U$) & dimension of $U$ & 5 & 4 \\
    $M$ & dimension of $\mathcal{D}_Q$ & 10000 & 10000 \\
    dim($X$) & dimension of $X$ & 13 & 12 \\
    dim($\tilde{X}$) & dimension of $\tilde{X}$ & 5 & 4 \\
    dim($X_C$) & dimension of $X_C$  & 6 & 4\\
    dim($X_U$) & dimension of $X_U$ & 2 & 4\\
    $d_{{\tilde{X}}}$ & poly. degree of $U$ used to generate ${\tilde{X}}$ & 2 & 1 \\
    $d_{Y}$ & poly. degree of $X, U$ used to generate $Y$ & 2 & 1 \\
    $d_{S}$ & poly. degree of $U$ used to generate $S$ & \{1, 2\} & \{1, 2\} \\
    $t_R$ & threshold on $R_Q - R_P$ used for $\beta_{(\cdot)}$ search & 0.03 & 0.03 \\
   $\sigma^*_{\text{logit}({\tilde{X}})}, \sigma^*_{\text{logit}(Y)}, \sigma^*_{\text{logit}(S)}$ & target std. dev. of ${\tilde{X}}$, $Y$, $S$ logits  & 2 & 2 \\
   $p_j$ & Bernoulli parameter for $U_j$, i.e. $p(U_j = 1)$ & $\sim \text{Unif}(0.2, 0.8)$ & --- \\
    $R_{ij}$ & correlation between variables $U_i$, $U_j$, $i \neq j$ & 0.1 & 0.1*\\
    $\sigma^2_{\tilde{X}}, \sigma^2_Y$ & variance of error terms $\varepsilon_{\tilde{X}}$, $\varepsilon_Y$ & 0.1 & 0.1\\
    $f_P$ & prediction model & XGBoost with 5-fold CV & Logistic regression with 5-fold CV \\

    \bottomrule
    \end{tabular}
    }
    \caption{Parameter settings for synthetic data experiments with binary and continuous data. Values enclosed in \{ \} indicates experiments were run with multiple settings for that parameter. XGBoost was implemented using \texttt{xgboost.XGBClassifer} and logistic regression was implemented using \texttt{sklearn.linear\_model.LogisticRegressionCV}. \\ {\small {*For $U \sim (\mathbf{0}, \mathbf{\Sigma})$}. For $U \sim \text{Uniform}(0, 1)$, the variables are $\mathrm{iid}$, so $R_{ij} = 0$. } }
    \label{appendix:tab:synth-params-binary}
\end{table}

\subsection{Additional details on All of Us data}\label{appendix:aou_data}
 As a broadly representative cohort with rich sociodemographic features, All of Us is well-suited for simulating selection bias and serves as a practical proxy for the target distribution.
 
\subsubsection{Data preprocessing}
Of the 628,664 All of Us participants, we remove participants who were missing medical records or sociodemographic information. First, we remove around 7,600 patients whose sociodemographic information included patient responses of "Skip" or "No matching concept". 
 We then filter all patients that had no recorded physical measurements (which were recorded once at participant registration), as well as participants who had no linked external medical records. In total, we collect four groups of variables: 4 mixed-type static sociodemographic variables collected at registration; 6 continuous physical measurements collected at registration; 11 survey questions from registration that largely represent lifestyle and social determinants of health; and 19 health outcomes as grouped OMOP codes from linked medical records. Variables were selected to have reasonable rates of missingness (<50\%)\footnote{Note that while the continuous variables were selected such that they are never missing, for the remaining variables we implicitly encoded missingness as another binary variable} and how relevant they are in predicting patient health. More information regarding All of Us data collection is available on the public data browser: \url{https://databrowser.researchallofus.org/}. 
 
 All OMOP health diagnoses were temporally aggregated into a single binary variable representing if the participant had every received the diagnosis. After categorical and continuous variables were binarized (continuous variables were transformed using quantile-based discretization with 5 buckets), we had a total dataset of $\dim(X) = 81$ variables. We list all variables (in non-binarized form) below: 

\textbf{Sociodemographic variables: }
\begin{itemize}[noitemsep]
    \item Race (merged with ethnicity)\footnote{For race, we created six categories: Asian, Black or African American, Mixed, White, and Hispanic or Latino. While these simple "racial categories" are often used for problematic social grouping and lack broader cultural and genetic contexts, for the purpose of this project the categories sufficed, and a deeper investigation into computational analysis beyond these groupings should be investigated.}
    \item Age
    \item Gender
    \item Sex
\end{itemize}
\textbf{Physical measurements: }
\begin{itemize}[noitemsep]
    \item Height
    \item Weight 
    \item Body mass index (BMI)
    \item Heart rate, mean of 2nd and 3rd measures
    \item Systolic blood pressure, mean of 2nd and 3rd measures
    \item Diastolic blood pressure, mean of 2nd and 3rd measures
\end{itemize}
\textbf{Survey questions: }
\begin{itemize}[noitemsep]
    \item Annual income 
    \item Self-reported health 
    \item Self-reported disability
    \item Cigarette consumption
    \item Has health insurance 
    \item Education level
    \item Alcoholic drink frequency
    \item Self-reported confidence filling out medical forms 
    \item Stable housing 
    \item Smokeless tobacco consumption 
    \item Employment status
    \item Primary language spoken$^*$
    \item Type of health insurance$^*$
    \item Marital status$^*$
\end{itemize}
where $^*$ denotes variables that were only used in the MIMIC-IV experiments. 

\textbf{Health conditions, based on OMOP code groups: }
\begin{itemize}[noitemsep]
    \item Type 2 diabetes mellitus
    \item Sleep apnea
    \item Hypertensive disorder
    \item Asthma
    \item Alcohol related disorders
    \item Multiple sclerosis
    \item Alzheimer's disease
    \item Chronic cardiovascular disease
    \item Fibromyalgia
    \item Chronic heart disease
    \item Endometriosis
    \item Chronic pain
    \item Human immunodeficiency virus (HIV)
    \item Drug dependence
    \item Substance abuse
    \item Depressive disorder
     \item Anxiety disorder
 \item Maternal / fetal condition affecting labor / delivery       
 \item Rheumatoid arthritis   
\end{itemize}

\subsubsection{Simulating EHR-specific selection bias}

We simulated EHR-specific selection bias by choosing 10 potential selection variables  as drawn from literature \citep{ehrbias1, ehrbias2, ehrbias3, ehrbias4} : race\footnote{Again, we recognize that race is proxy for many unobserved features that actually drive selection, such as historical trust in medical institutions.}, self-reported medical confidence, education level, annual income, employment status, stable housing, health insurance, self-reported health, and age. For nonlinear selection mechanisms, we randomly selected a subset of dim$(U)$ variables, with a preference for selecting on age given evidence of hospitals over-representing older adults (as a proxy for increased healthcare needs) \citep{kim2024bias}. The weights in the nonlinear selection mechanism (a degree-2 polynomial) were randomly sampled and appropriately scaled as detailed in the fully synthetic case. For the linear selection mechanism, we again randomly sampled selection variables but fixed each variable's corresponding weight to match the expected (univariate) effect of selection bias on the observed data. For example, the binary variables \texttt{has\_health\_insurance} and \texttt{has\_stable\_housing} had $\beta=1$ (so patients with these variables were oversampled in $P$), whereas \texttt{has\_low\_medical\_confidence} had $\beta=-1$ (so patients with low confidence were undersampled in $P$).

\subsubsection{Task selection}

We chose to predict patient diagnosis from the known comorbidity, lifestyle, and social determinants of health variables. Explicitly, each prediction task selected $Y$ from the 19 health outcomes coded by OMOP codes and the covariates $\tilde{X}$ as a random subset of the remaining covariates. Tasks were selected similarly to that described in Algorithm \ref{alg:synth_sel} such that $R_Q - R_P > t_R$ for some threshold $t_R$.

For the linear selection mechanism, we manually designed realistic prediction tasks, selecting $\dim(U)=5$ selection variables with the corresponding EHR-specific selection weights, health outcome $Y$, and covariates $\tilde{X}$ focused on lifestyle and physical measurements. For both selection mechanisms we followed the data-driven search as outlined in Algorithm \ref{alg:synth_sel} to select prediction tasks (minus weight sampling for the linear selection mechanism). For each of the $30$ tasks, we repeated over $20$ seeds where we resampled the data with each seed. Explicitly, we only selected tasks $(\tilde{X}, Y)$ and selection mechanism $g(U\beta)$ where the logloss generalization gap was at least 0.03. Both mechanisms learned $f_P$ using \texttt{sklearn.linear\_model.LogisticRegressionCV} with class-balancing weights. 

\subsection{Additional details on MIMIC-IV data}\label{appendix:mimic_data}

\subsubsection{Data preprocessing}
MIMIC-IV data \citep{johnson2023mimiciv, johnson2024mimiciv, goldberger2000physionet} is made publicly available for research and educational purposes under PhysioNet Credentialed Health Data License 1.5.0. To help preprocess the data, we used the MIMIC-IV Data Pipeline courtesy of \cite{mimicCode}. Our work used hospital data pulled from MIMIC-IV version 3.1. Outlier data was imputed for continuous values at the 1st and 99th percentile, and time-series data was compressed to static.  

Variable selection from MIMIC-IV was first driven to select overlapping variables with All of Us, as well as the variables with highest observation frequency:

\textbf{Sociodemographic variables:}
\begin{itemize}[noitemsep]
    \item Age 
    \item Gender 
    \item Race (combined with ethnicity)
    \item Type of health insurance
    \item Marital status
    \item English as primary language
\end{itemize}
\textbf{Physical measurements:}
\begin{itemize}[noitemsep]
    \item Height 
    \item Weight
    \item Systolic blood pressure 
    \item Diastolic blood pressure 
\end{itemize}

\textbf{Lab measurements:}
\begin{itemize}[noitemsep]
    \item Potassium levels$^*$
    \item Anion gap$^*$
    \item Sodium levels$^*$
    \item Urea nitrate levels$^*$
    \item Platelet count$^*$
    \item Hematocrit levels$^*$
    \item Bicarbonate levels$^*$
    \item Glucose levels$^*$
    \item Chloride levels$^*$
    \item Creatine levels$^*$
\end{itemize}
where $^*$ denotes values that we observe in MIMIC-IV but not All of Us\footnote{These values exist in All of Us but were not collected given they are highly unlikely to be root drivers of selection.}, and thus are non-candidates for potential $U$. 

\textbf{Health conditions, based on ICD10 code groups: }
\begin{itemize}[noitemsep]
    \item Type 2 diabetes mellitus
    \item Hypertensive disorder
    \item Alcohol related disorders
    \item Multiple sclerosis
    \item Alzheimer's disease
    \item Chronic cardiovascular disease
    \item Fibromyalgia
    \item Chronic heart disease
    \item Endometriosis
    \item Chronic pain
    \item Human immunodeficiency virus (HIV)
    \item Depressive disorder
     \item Anxiety disorder
 \item Rheumatoid arthritis   
    \item Mortality within the next 8 hours$^*$

\end{itemize}

Next, we had to align variables across the two datasets. Specific details for mapping categories across health insurance, gender, race, and primary language variables are available upon request. To prevent overlap issues, continuous variables were clipped to the maximum and minimum of that variable in All of Us. Continuous variables were separately binned across datasets to 3 equal sized bins. We also removed categorical classes that had negligible frequency in MIMIC-IV. Finally, since diagnoses are coded using ICD10CM in MIMIC-IV and using OMOP in All of Us, we mapped OMOP codes defining the above groups to the ICD ontology using the \texttt{ancestor} table in OMOP.  

\subsubsection{Selection mechanism}
Although the true selection mechanism is unknown, we hypothesize that selection from the more representative All of Us population into MIMIC's Beth Israel Deaconess Medical Center population are likely driven by age, insurance type, or primary language (as a proxy variable for medical trust), which are fully observed in both datasets and thus serve as the potential set of selection variables $\tilde{U}$. 
\subsubsection{Task selection}

We run two types of prediction task experiments using MIMIC-IV. In the first type of task, we can leverage our oracle knowledge of $Q$ samples provided by the All of Us data. If the covariates $\tilde{X}$ and outcome $Y$ are fully observed in both datasets, then we can actually compare our bound estimate $\hat{R}_Q$ to the true $R_Q$. In the second type of task, we test the more realistic usage of our method where we extrapolate to prediction tasks that involve covariates and/or health outcomes we observe only in $P$ and not in $Q$. 

To test the accuracy of our estimate $\hat{R}_Q$ relative to the true $R_Q$, we randomly select a subset of binary variables $\dim(\tilde{U}) \in [2,4]$ from the three aforementioned categorical variables; $\dim(\tilde{X})=5$ covariates that are observed in both datasets; and a health condition $Y$ that is observed in both datasets. We selected two tasks such that the logloss generalization gap is at least 0.03. Each task is repeated for 10 seeds by refitting a \texttt{sklearn.linear\_model.LogisticRegressionCV} with class-balancing weights as the prediction model $f_P.$

For the second type of prediction task, we extend our method to its full intent: estimating the generalization gap under limited $Q$ observability. Specifically, we set the outcome variable to be the $Y=$Mortality (within 8 hours), and we include as covariates several lab observations that are not present in the All of Us dataset. For this type of task, we randomly selected $\dim(\tilde{X})=8$ covariates and $\dim(\tilde{U})=4$ observed selection variables.


We present the full variable details for these tasks in Table \ref{table:mimic_detailed}.

\begin{table}[htbp]
\centering
\renewcommand{\arraystretch}{1.3} 
\begin{adjustbox}{max width=\textwidth}
\begin{tabular}{@{}l p{0.3\textwidth} p{0.3\textwidth} p{0.3\textwidth}}
\toprule
& \textbf{Y = Hypertension} & \textbf{Y = T2DM} & \textbf{Y = Mortality} \\
\midrule
\textbf{Covariates $\tilde{X}$} & Diastolic blood pressure (1), comorbidities (substance abuse, Type II diabetes mellitus, Alzheimer's), gender & Marital status, race, primary language, comorbidity (Anxiety disorder) & Weight (2), comorbidites (Type II diabetes mellitus, endometriosis, hypertension), platelet counts$^*$  (1), insurance type (1), bicarbonate levels$^*$ (1), systolic blood pressure (1) \\
\textbf{Target $Y$}   &  Hypertensive disorder &  Type II diabetes mellitus & Mortality within the next 8 hours$^*$      \\
\textbf{Known $\tilde{U}$}          & Age (2), English as primary language & Age (1), English as primary language, insurance type (1) & Age (1), insurance type (3)        \\
\textbf{Predicted $U^*$}     & Comorbidities, diastolic blood pressure (1), systolic blood pressure (1), insurance type (3), race (3), marital status  (2)   &  Comorbidities, age (1), diastolic blood pressure (1), systolic blood pressure (1), insurance type (2), race (2), marital status (2)  & Comorbidities, age (1), diastolic blood pressure (1), insurance type (1), race (3), marital status (2)      \\
\bottomrule
\end{tabular}
\end{adjustbox}
\caption{Details for the three real-world experiments in Figure \ref{fig:mimic} that compared generalization prediction performance of MIMIC-IV as the observed $P$ to All of Us as the target $Q$. The parentheses and number $(\cdot)$ denote the number of binary variables chosen for that continuous or categorical variable. $^*$ denotes variables unobserved in All of Us.}\label{table:mimic_detailed}
\end{table}

\subsection{Simulating selection} We present the general algorithm for our synthetic selection experiments in Algorithm \ref{alg:synth_sel} as run in the fully synthetic and All of Us experiments. At a high level, we first construct an ``oracle" target dataset $\mathcal{D}^*_Q$ containing all of the variables of interest: $X, Y, U, S$. This lets us compute the ground truth $R_Q$, which serves as a point of reference and sanity check against our estimated bound, $\hat{R}_Q$. When we compute our bound, we rely on only a partially observed target dataset, $\mathcal{D}_Q$, derived from the oracle dataset by excluding $X, U^*, Y$ samples, and summary statistics for all variables. As outlined in Section \ref{sec:experiments:sel-pred-mdl}, the biased dataset, $\mathcal{D}_P$, is constructed by sampling the points from $\mathcal{D}^*_Q$ where the selection indicator $S = 1$.

To generate the selection indicator $S$, we sample $S \sim \text{Bernoulli}(1/(1+\text{exp}(-g(U\beta_S)))$, as per the model introduced in Section \ref{sec:experiments:sel-pred-mdl}. We define the link function $g(U\beta_S) = \phi^{(d_S)}(U)\beta_{U \rightarrow S} + \beta_{0, S}$ where $\phi^{(d_S)}$ is the degree $d_S$ polynomial expansion of $U$. Thus, our \textbf{linear selection} experiments simply uses a degree 1 polynomial and \textbf{nonlinear selection} a degree 2 polynomial. 

We want to choose selection coefficients $\beta_S$ such that the result of applying the coefficients has a distribution invariant to the number of features. Thus, after we first sample initial values for $\beta_S$, we then normalize the $\beta_S$ vector to achieve a specific target standard deviation. Consider the desired polynomial expansion $\phi^{(d_S)}$ which transforms $U$ to some $p$ features, and a normally sampled $\beta_{U \rightarrow S} \sim N(\textbf{0, I}_p)$. Let the resulting standard deviation of the logits of $S$ be defined as 
\begin{align*}
    \sigma_{\text{logit}(S)} := \text{stddev}(g(U\beta_{U \rightarrow S})) = \text{stddev}(\phi^{(d_S)}(U)\beta_{U \rightarrow S} + \beta_{0, S})
\end{align*}
Given a target standard deviation $ \sigma^*_{\text{logit}(S)}$, we then scale the weights to achieve the desired standard deviation regardless of dim($U$): 
\begin{align*}
(\beta_{U \rightarrow S}, \beta_{0, S}) 
        &\leftarrow
        (
        \frac{\beta_{U \rightarrow S} \cdot \sigma^*_{\text{logit}(S)}}
        { \sigma_{\text{logit}(S)}},
        \frac{\beta_{0, S}\cdot \sigma^*_{\text{logit}(S)}}{ \sigma_{\text{logit}(S)}}
        )
\end{align*}
As described in Section \ref{appendix:dgp_synth}, we apply a similar sampling and scaling procedure for the $\beta$ coefficients that are used to generate $U$, $X$, and $Y$.

\begin{algorithm}[t]
\caption{General Algorithm for Synthetic Selection}
\label{alg:synthetic-selection}
\begin{algorithmic}[1]
\Require dim($\tilde{X}$), dim($U$), dim($\tilde{U}$), logistic link function for selection model $g$, gap threshold $\mathsf{thresh}$, prediction model $f_P$, number of prediction tasks $n_{\mathrm{tasks}}$, number of seeds $n_{\mathrm{seeds}}$,  \texttt{mode} $\in \{\texttt{synthetic},\texttt{real}\}$, params for Algorithm~1; 

\textbf{if mode==synthetic:} params for generating $(X,Y,U)$;

\textbf{if mode==real:} variable lists for potential $(\tilde{X},U,Y)$, access to $\mathcal{D}^*_Q$.
\Ensure $\hat{R}_Q, R_Q, R_P$ across $n_{\text{tasks}}$ and $n_{\text{seeds}}$

\Statex \Comment{Iterate over prediction tasks}
\For{$t \gets 1$ to $n_{\mathrm{tasks}}$}

\Statex \Comment{Step 0: Construct and sample $(\tilde{X}, Y, U) \sim Q$}
\If{\texttt{mode} = \texttt{synthetic}}
  \State Generate $(\tilde{X},Y,U, X-\tilde{X}) \sim Q$ per (Appendix \ref{appendix:dgp_synth});
\Else
  \State Sample from input variable lists and then $\mathcal{D}^*_Q$ to construct $(\tilde{X},U,Y) \sim Q$
\EndIf

\State Sample and scale coefficients $\beta_S :=(\beta_{U\rightarrow S}, \beta_{0, S})$ for the selection model as described above

\Statex \Comment{Iterate over seeds}
\For{$s \gets 1$ to $n_{\mathrm{seeds}}$}

  \Statex \Comment{Step 2: Induce selection to form $\mathcal{D}_P$}
  \State Sample $S \sim \text{Bernoulli}(1/(1+\text{exp}(-g(U\beta_S)))$, where $g$ may be \textit{linear} or \textit{nonlinear}.
  \State Draw the selection biased sample $\mathcal{D}_P = \{(X^{(i)}, U^{(i)}, Y^{(i)}): S^{(i)} = 1\}$.

  \Statex \Comment{Step 3: Get $R_Q$, $R_P$}
  \State Fit $f_P:\tilde{X}\!\to\!Y$ on $\mathcal{D}_P$ to obtain $R_P$ 
  \State Apply $f_P$ to $\mathcal{D}^*_Q$ to obtain $R_Q$

  \Statex \Comment{Step 4: Test if exceeds threshold}
  \If{$R_Q - R_P > \mathsf{thresh}$}
    \State Proceed with bound estimation
  \Else
    \State \textbf{continue} \Comment{Reject and move to next seed}
  \EndIf

  \Statex \Comment{Step 5: Compute bound estimate $\hat{R}_Q$}
  \State Select observed $\tilde{U} \subset U$
  \State Run \textbf{Algorithm} \ref{alg} using learned $f_P$, $\mathcal{D}_P$, $\mathcal{D}_Q$

  \State Save $\hat{R}_Q, R_Q, R_P$ for seed $s$ and task $t$.
\EndFor
\EndFor
\State \Return $\hat{R}_Q, R_Q, R_P$ across $n_{\text{tasks}}$ and $n_{\text{seeds}}$
\end{algorithmic}\label{alg:synth_sel}
\end{algorithm}

\section{Additional experimental setup }

\subsection{Testing for assumption violations in practice}\label{appendix:ass_viol_practice}

We note that deriving a statistical test for conditional independence or common support is challenging even in the fully observed setting \cite{shah2020hardness,gretton2012kernel}, in contrast to our setting of partial observability (Constraints \ref{ass:no_samples} - \ref{ass:mu_sigma}) where it is impossible to prove these assumptions. While statistically rigorous tests for these assumptions would be broadly useful and certainly warrant further investigation, it is out of scope for this work. Instead, we provide approximate diagnostics that are practical, straightforward, and tractable for practitioners to assess for conditional independence and common support in real-world settings.

\subsubsection{Common Support Diagnostic: KS Test}
Testing positivity via propensity score distributions has its origins in causal inference literature \cite{stuart2013prognostic,petersen2012diagnosing}, where differences in covariate propensity between treated and untreated groups are often evaluated using t-tests or Kolmogorov–Smirnov (KS) tests. In our setting, where "treatment" is selection $S$ into the biased distribution $P$, testing for common support of the true propensity $p(S=1 \mid U)$ is ideal but unrealistic given $U$ is unknown and unobserved in $Q$. As a result, we employ the next-best option where we compare the propensity distributions of $P$ and $Q$ using the observed selection variables $\tilde{U}$ via a KS test. Consequently, the richer the set of $\tilde{U}$ observed in the target dataset $\mathcal{D}_Q$, the more valid the test for assessing common support. Conditioned on user’s assessment of the target dataset quality (i.e., how many features in $\mathcal{D}_Q$ match the expected factors causing selection?), statistical validity of this test can be assessed based on the resulting p-value. Although we employ a KS test, other statistical tests could work (i.e., a t-test).
\subsubsection{Common Support Diagnostic: Weight Design Effect} The success of importance weighting methods such as moment-matching is often evaluated by analyzing the variance of the resulting weights \cite{ipw3,markoulidakis2023tutorial,chatrchi2015impact}. While not formal hypothesis tests with p-values, metrics like design effect ($d_{\text{eff}}$, i.e., using \cite{kish1992weighting}) and effective sample size (i.e., ESS = $n / d_{\text{eff}},$ where $n$ is the original sample size) can diagnose unstable weights and indicate lack of common support that underlies the weighting objective (i.e., \citet{cole2008constructing} applied similar logic to propensity weights). A common practice considers weights unstable when $ESS \ll \rho \cdot n$ or equivalently $d_{\text{eff}} \gg 1/\rho,$ where $\rho \in [0,1]$ is the user’s level of tolerance. In our setting, for example, a more conservative user could determine that approximate common support between $P$ and $Q$ occurs when $ESS \ll 0.1 \cdot |\mathcal{D}_P|$, where ESS is computed from the moment-matching weights $1/g(\tilde{U} \omega  + C \gamma)$. We acknowledge that while common support could be one reason the moment-matching method fails, there could be other reasons such as model misspecification.

 \subsubsection{Conditional independence diagnostic: Propensity Invariance} To the best of our knowledge, we propose a novel test for approximate conditional independence by checking for equality across $$p(V\mid \hat{U},\hat{S}=s_1) = p(V\mid \hat{U},\hat{S}=s_2);   \ \ \forall s_1,s_2\in [0,1], \ \ \forall V\in (X,Y)$$ where $\hat{S}$ is defined as the observable propensity $p(S=1 | \tilde{U})$ and $\hat{U} := (\tilde{U} , \hat{U}^*)$ is our predicted $U$ from our moment-matching method. To implement this test, samples from $\mathcal{D}_P$ are partitioned based on their propensity value falling into discrete bins $\hat{S} \in [\alpha, \beta]$. For each propensity set, we train a classifier to learn $p(V \mid \hat{U}, \hat{S} \in [\alpha, \beta])$ and then evaluate the classifier’s performance on all other propensity sets. Similar performance across all propensity sets and for all $V \in (X,Y)$ indicate greater confidence in the conditional independence assumption. To compress this test into a single metric, we compute the variance of a classifier’s logloss performance for all propensity sets, and report the average variances across all classifiers and all $V \in (X,Y)$. As with the tests above, the validity of this method depends on several unknowns: first, the accuracy of $\hat{U}$ as $U$; secondly the proximity of the observed propensity $p(S=1 \mid \tilde{U})$ to the true propensity given $U$.

\subsubsection{Experimental Setup}
We describe the experiment reported in Table \ref{tab:violation_tests}. The experiment was run on fully synthetic data using the parameter settings previously described, with the exception of varying sample size $|\mathcal{D}_Q| \in\{1000,10000\}$, selection strength $\sigma^*_{\text{logit}(S)} \in \{1.25 \ (\text{Low}), 2 \ (\text{High})\}$, fixing a linear selection mechanism, and running over all possible $\tilde{U}\subseteq U$, $n_{\text{tasks}}=10$, and 2 seeds. For each experiment, we calculated the \textit{KS test} using a logistic regression model to fit $p(S=1 \mid \tilde{U})$ and report the corresponding KS statistic (p-value could also be used); \textit{Weight Design Effect} as $d_{\text{eff}} = \frac{n \sum_{i=1}^n (w^{(i)})^2}{(\sum_{i=1}^n w^{(i)})^2}$ where $w^{(i)}$ is the corresponding weight from our moment-matching heuristic and $n = |\mathcal{D}_P|$ \cite{kish1992weighting}. To estimate \textit{Propensity Invariance}, we first estimate the predicted selection variables $\hat{U}:=(\tilde{U},\hat{U}^*)$. Let $W :=(X,Y) \setminus \hat{U}$ be all remaining variables that are not predicted to be causing selection. Using the aforementioned propensity model learned over $\mathcal{D}_Q$ and $\mathcal{D}_P$, we consider all unique propensity scores $s_j = p(S=1 \mid \tilde{U}=u)$ observed in $\mathcal{D}_P$ (which is feasible given we are working with discrete $\tilde{U}$). For each score and its corresponding dataset $\mathcal{S}_j = \{(\tilde{X}^{(i)},Y^{(i)}) \in \mathcal{D}_P : p(S=1 \mid \tilde{U}^{(i)}) = s_j\}$, and then for each variable $V \in W$, we fit a logistic regression model $f_{j}$ over the training set $\mathcal{S}_j$ to predict $V$ given $\hat{U}$. We then estimate the out-of-distribution variance on other sets defined by different propensity scores:  $\sigma^2_j(V) = \text{Var}_{k: \ s_k}(\text{logloss of } f_j \text{ on } \mathcal{S}_k)$. We then average all $\sigma^2_j(V)$ over all propensity scores $j$ and all variables $V$. Note, for all metrics, a higher score indicates higher (approximate) violation of that corresponding assumption. 

\subsection{Additional details on baselines}\label{appendix:baselines}
We tested our bound against the following baselines: inverse probability of participation weighting (IPPW), empirical calibration, entropy balancing, and raking (iterative proportional fitting). In this section, we describe each of these baselines in greater detail, including both a high-level description of the theory underlying each baseline and any hyperparameter settings used in our implementations. We note that all four baselines are reweighting methods, meaning the estimated $\hat{R}_Q^{(baseline)}$ produced by each method can be expressed as
\begin{align*}
    \hat{R}_Q^{(baseline)} = \mathbb{E}_P[w^{(baseline)}(\tilde{X}, U, Y)\cdot \ell(f_P(\tilde{X}), Y)]
\end{align*}
where $w^{(baseline)}(\tilde{X}, U, Y)$ are the per-sample weights estimated by a given baseline.

\paragraph{Inverse probability of participation weighting (IPPW) } IPPW \citep{ipw2000} has been used in prior works to address sample selection bias \citep{ipw1,ipw2,ipw3,ipw4}. In IPPW, data points are reweighted by a ratio between the probability of observing the particular point in the target population and the probability of observing that point in the biased sample, in order to ``align" the biased distribution with the target distribution. More concretely, let $\tilde{U}$ denote the set of known and observed selection variables, implying that the probability of observing a data point for a particular $\tilde{U}$ value is different in the biased sample than in the target population. We then define the weight\footnote{The first term is also equivalent to post-stratification weighting \citep{ipw1}.} $w^{(IPPW)}(\tilde{U}) = p(\tilde{U}) \ / \ {p(\tilde{U} \mid S=1)} = p(S=1) \ / \ p(S=1 \mid \tilde{U})$. Estimating $p(\tilde{U}) \ / \ {p(\tilde{U} \mid S=1)}$ can be done through estimating the densities $p(\tilde{U})$ and $p(\tilde{U} \mid S=1)$ or through estimating the density ratio itself directly \citep{kliep}. Alternatively, estimating $p(S=1) \ / \ p(S=1 \mid \tilde{U})$ can be accomplished via classification-based methods \citep{ipw1}. In our implementation, we calculate $w(\tilde{U})$ by estimating the density functions $p(\tilde{U})$ and $p(\tilde{U} \mid S = 1)$.

Notice that this $w^{(IPPW)}(\tilde{U})$ is the same as the $w(\tilde{U})$ weight that appears in our proposed bound. If $\tilde{U}$ is the complete set of selection variables, and there is common support between $P$ and $Q$, then IPPW produces an unbiased estimate of $R_Q$. When there are other selection variables $U^*$ that are not included in the weights, then IPPW yields a biased estimate for $R_Q$. Ideally, we would want to recover the unbiased estimate by calculating the IPPW weights for the selection variables: $w(\tilde{U}, U^*) =  p(S=1) \ / \ p(S = 1 \mid \tilde{U}, U^*)$. However, with only partial observability of $U^*$ in $Q$ (see Constraints \ref{ass:no_samples}  - \ref{ass:mu_sigma}), we cannot easily estimate $w(\tilde{U}, U^*)$. Thus, we use IPPW with $w(\tilde{U})$ weights as a realistic baseline $R_Q$ estimate. Similar assumptions about using only observed selection variables in estimating IPPW weights were made in prior works on empirical selection bias correction, such as \citep{ipw1}. 

\paragraph{Empirical calibration} Empirical calibration \citep{surveyest} solves for weights that balance covariates across the biased and target datasets, while simultaneously controlling variance by making sure the weights do not stray too far from uniform weights. Empirical calibration can be expressed as a constrained optimization problem where the goal is to minimize an objective function characterizing deviation from ``uniformity" of weights while meeting certain balance or ``calibration" constraints.

For our problem setting, we implement a form of empirical calibration with an entropy loss function and with first moment balance constraints. Roughly, this formulation seeks to match the weighted average of the variables observed in $\mathcal{D}_P$ with the variable averages observed in $\mathcal{D}_Q$. More specifically, we solve the following optimization problem to obtain the weights $w^{(\text{calib})}$:
\begin{align*}
    &\min_{w} \sum_{i=1}^{|\mathcal{D}_P|}w^{(i)} \log(w^{(i)}) \\
     & \text{s.t.} 
     \sum_{i=1}^{|\mathcal{D}_P|} w^{(i)} X_{j}^{(i)} = \mu_Q(X_j) , \ \ \sum_{i=1}^{|\mathcal{D}_P|} w^{(i)} U_{k}^{(i)} = \mu_Q(U_k) \\ & \quad\text{for all $j = 1, \dots, \text{dim}(X), \ k=1, \dots, \text{dim}(U)$} \\
    & \sum_{i=1}^{|\mathcal{D}_P|} w^{(i)} = 1 \\
    & w^{(i)} \geq 0 \text{ for all } i = 1, \dots |\mathcal{D}_P|
\end{align*}
where the notation $X_j^{(i)}$ refers to the $i$-th sample of variable vector $X$, and the $j$-th variable element within that vector. Notice that we include \textit{all} $X$ and $U$ in the balance constraints rather than just $\tilde{U}$ (the known drivers of selection) which reflects how empirical calibration is often applied in practice. Additionally, note that, while we observe $\tilde{U}_j$ fully in $Q$, our empirical calibration implementation collapses this knowledge down to the first moment, $\mu_Q(\tilde{U}_j)$. Thus, when the mean is not a good representation of a particular distribution, e.g. for a bimodal distribution, then balancing on first moments alone can produce weights that do not sufficiently correct for selection bias. 

We use the python implementation provided by \cite{wang2019python}, where we run approximate calibration using the library's \texttt{maybe\_exact\_calibrate} function with the objective set to \texttt{objective.ENTROPY} and covariates scaled to be between 0 and 1. We assume that observed summary statistics for continuous variables under Constraint \ref{ass:mu_sigma} are derived from the raw continuous values themselves rather than from binned versions of the continuous variables  \footnote{For heart rate and BMI, for example, we are more likely to observe the average continuous values as opposed to average the discrete binned values}. Thus, we run all calibration baselines (and our heuristic $U^*$ algorithm) on non-binned, continuous summary statistics for any variables that are continuous, regardless of whether these variables are later binned for density estimation purposes.

\paragraph{Entropy balancing} Entropy balancing is a specific type of empirical calibration that uses an entropy loss function and that, in typical implementations, solves for weights that balance both first and second moments. For our problem setting, we implemented a form of entropy balancing that solves the following optimization problem to obtain the weights $w^{(\text{eb})}$:

\begin{align*}
    &\min_{w} \sum_{i=1}^{|\mathcal{D}_P|}w^{(i)} \log(w^{(i)}) \\
     \text{s.t.} 
     & \sum_{i=1}^{|\mathcal{D}_P|} w^{(i)} X_{j}^{(i)} = \mu_Q(X_j) , \sum_{i=1}^{|\mathcal{D}_P|} w^{(i)} U_{k}^{(i)} = \mu_Q(U_k) \\ 
     & \sum_{i=1}^{|\mathcal{D}_P|} w^{(i)} (X_{j}^{(i)})^2 = \sigma^2_Q(X_j) + \mu_Q(X_j)^2, \sum_{i=1}^{|\mathcal{D}_P|} w^{(i)} (U_{k}^{(i)})^2 = \sigma^2_Q(U_k) + \mu_Q(U_k)^2
     \\
     &\quad\text{for all $j = 1, \dots, \text{dim}(X), \ k=1, \dots, \text{dim}(U)$} \\
    & \sum_{i=1}^{|\mathcal{D}_P|} w^{(i)} = 1 \\
    & w^{(i)} \geq 0 \text{ for all } i = 1, \dots |\mathcal{D}_P|
\end{align*}
where note we simplify the second moment alignment by using the definition of variance: $\mathbb{E}[(V - \mathbb{E}[V])^2] = \mathbb{E}[V^2] - \mathbb{E}[V]^2$. 

Like with empirical calibration, entropy balancing collapses the known information for $\tilde{U}$ down to first and second moments, thus potentially omitting key distributional properties and leading to weights that do not properly correct for selection bias.

Again, we use the python implementation provided by \cite{wang2019python}, where we run approximate calibration using the library's \texttt{maybe\_exact\_calibrate} function with the objective set to \texttt{objective.ENTROPY} and covariates are scaled to be between 0 and 1. 

\paragraph{Raking (Iterative proportional fitting)} Raking \citep{surveyest}, equivalently known as iterative proportional fitting, aligns categorical sample data to univariate table counts in the target data through an iterative adjustment process. First, for sake of illustration, suppose there are two scalar categorical variables, $X$ and $U$, that can each take on levels $\{1, 2, 3, \dots, L_X\}$ and $\{1, 2, 3, \dots, L_U\}$, respectively. Suppose the joint counts of $X$ and $U$ in $\mathcal{D}_P$ are stored in an $L_X \times L_U$ matrix, $\mathbf{M}$ such that the value $M_{xu}$ in row $x$, column $u$ contains the number of occurrences of $(X, U) = (x, u)$ in the observed data $\mathcal{D}_P$. That is, $M_{xu} = \sum_{i=1}^{|\mathcal{D}_P|}\mathbbm{1} \{X^{(i)} = x\} \cdot \mathbbm{1} \{U^{(i)} = u\}$.

The goal is then to solve for weight vectors that align the marginal row and column counts (or frequencies) in the biased sample with those in a target sample. Concretely, we solve for $\boldsymbol{\beta}$ and $\boldsymbol{\phi}$ such that $\sum_{x=1}^{L_X} \beta_x \phi_uM_{xu} = \mu_Q(u)$ for all $u = 1, \dots, L_U$ and $\sum_{u=1}^{L_U} \beta_x \phi_uM_{xu} = \mu_Q(x)$ for all $x = 1, \dots, L_X$ where $\mu_Q(u)$ denotes the marginal probability $p(U = u)$ in the target population and $\mu_Q(X)$ denotes the marginal probability $p(X = x)$, which are both observed under Constraint \ref{ass:mu_sigma}. Solving for these weights is usually done iteratively by first aligning row marginals, then aligning column marginals, and repeating until the weighted marginals in the biased sample are sufficiently close to those in the target population. Since raking assumes the data are categorical, continuous data must first be binned before applying raking. The iterative alignment process can be generalized to the case where $X$ and $U$ are multivariate with dimension $\dim(X) > 1$, $\dim(U) > 1$. Assume, without loss of generality, that each variable $X_j, \ j = 1, 2, \dots, \dim(X)$ takes on possible values $X_j \in \{1, 2, \dots, L_{X_j}\}$, and each $U_k, \ k = 1, 2, \dots, \dim(U)$ takes on possible values $U_k \in \{1, 2, \dots, L_{U_k}\}$. The output of the alignment produces two sets of alignment weight vectors, $\{\boldsymbol{\beta}_j\}_{j=1}^{\dim(\mathbf{X})}$ and $\{\boldsymbol{\phi}_k\}_{k=1}^{\dim(\mathbf{U})}$ where $\dim(\boldsymbol{\beta}_j) = L_{X_j}$ and $\dim(\boldsymbol{\phi}_k) = L_{U_k}$ for all $j$ and $k$. 

Recall that raking, like the other baselines, is a reweighting method for estimating $\hat{R}_Q$. Let $w^{(raking)}(X^{(i)}, U^{(i)})$ denote the per-sample weight applied to the $i^{\text{th}}$ sample in $\mathcal{D}_P$ in the estimation of $\hat{R}_Q^{(raking)}$. We compute $w^{(raking)}(X^{(i)}, U^{(i)})$ using the alignment weight vectors as follows:
\begin{align*}
    w^{\text{(raking)}}(X^{(i)}, U^{(i)}) = \prod_{j=1}^{\dim(X)}\prod_{k=1}^{\dim(U)} \beta_{j, X_j^{(i)}} \cdot \phi_{k, U_k^{(i)}} 
\end{align*}
where $\beta_{j, X_j^{(i)}}$ denotes the $X_j^{(i)}$-th element of $\boldsymbol{\beta}_j$ and $\phi_{k, U_k^{(i)}}$ denotes the $U_k^{(i)}$-th element of $\boldsymbol{\phi}_k$. 

In our implementation, we use the \texttt{weightipy} library in python. We drop continuous columns\footnote{We also tried binning continuous variables and did not observe a significant change in performance.} and then align the discrete variables using \texttt{weightipy.weight\_dataframe(P\_dataframe, Q\_scheme)} on the observed $\mathcal{D}_P$ samples as \texttt{P\_dataframe} and the dictionary of marginal frequencies from $Q$ as the scheme \texttt{Q\_scheme}.

\subsection{Density estimation}
\subsubsection{Practically estimating densities in the proposed bound}
To estimate $p(\tilde{U})$ and $p(\tilde{U} \mid S = 1)$ (used to compute $w(\tilde{U})$ in the bound), any joint density estimation (either counts for discrete data or continuous density estimation methods like KDE) can be applied directly. To estimate the conditional density functions $p(\tilde{X}, Y, \tilde{U} \mid S=1 )$ and $p(\tilde{X}, Y ,\tilde{U}, U^* \mid S=1)$ (used to compute $\phi(\tilde{X}, Y, \mathcal{U}*)$ in the bound), we first use the chosen density estimation method to estimate $p(\tilde{U}, U^* \mid S = 1)$, $p(X, Y, \tilde{U} \mid S=1)$, and $p(X, Y, \tilde{U} ,U^* \mid S = 1)$. From these density functions, we then compute the conditional density functions in our proposed bound as 
\begin{align*}
    \hat{p}(\tilde{X}, Y \mid \tilde{U}, S=1 ) &= 
    \frac{\hat{p}(\tilde{X}, Y , \tilde{U} \mid S = 1)}
    {\hat{p}(\tilde{U} \mid S = 1)}   \\
    \hat{p}(\tilde{X}, Y \mid \tilde{U}, U^*, S=1 ) &= 
    \frac{\hat{p}(\tilde{X}, Y, \tilde{U}, U^* \mid S = 1)}
    {\hat{p}(\tilde{U}, U^* \mid S = 1)}
\end{align*}
where the superscript $\hat{}$ denotes an estimated density.
\subsubsection{Continuous density estimation}\label{appendix:cont_dens_est}
Computing our proposed bound requires estimating several density functions, namely $p(\tilde{U})$, $p(\tilde{U} \mid S = 1)$, $p(\tilde{X}, Y \mid \tilde{U}, \hat{U}^*, S = 1)$ and $p(\tilde{X}, Y \mid \tilde{U}, S = 1)$. For low-dimensional binary or categorical data, densities can be estimated in a straightforward manner by computing simple counts. For continuous data, more sophisticated density estimation techniques are needed, including non-parametric methods like kernel density estimation (KDE) \citep{Rosenblatt1956KDE,parzen1962estimation} and semi-parametric methods like Gaussian mixture modeling (GMM) \citep{Li1999GMM}. Even with these methods, estimating continuous densities accurately is notoriously challenging and can come with large computational costs, especially when working with higher-dimensional data. While more modern, machine learning-based methods such as conditional normalizing flows \citep{flow1, flow2} can better handle higher-dimensional data, these methods require large sample sizes. 

In Appendix \ref{appendix:continuous_results}, we offer a preliminary assessment of how our proposed bound fares in continuous data settings using synthetic, low-dimensional continuous data and applying KDE to estimate densities. We recognize the need for future work to explore density estimation for higher-dimensional continuous data and potential workarounds such as estimating density ratios instead of densities directly, which can be less computationally expensive and more stable.  

For the continuous data experiments, where $U$ and $X$ are continuous multivariate random variables, we estimate the density functions needed to compute our proposed bound using kernel density estimation (KDE). We implemented KDE using \texttt{skearn.neighbors.KernelDensity} with a Gaussian kernel. To tune the bandwidth parameter for KDE, we performed a search over a candidate set of 10 logarithmically spaced points between 0.1 and 1 and chose the parameter that produces the best 3-fold cross validation score.




\section{Results}\label{appendix:results}
\subsection{Additional 
results for synthetic experiments}\label{appendix:synth-add-res}
\subsubsection{Validating our bound}
In Table \ref{tab:bound-err-synth-lin}, we report summary metrics on the error of our proposed bound with a linear selection mechanism, compared to the error of baseline estimates.

We highlight a few observations. First, although baselines that rely on marginal variable alignment (e.g., entropy balancing) perform reasonably well under the simpler linear selection mechanism, they degrade under the more complex nonlinear process, with only 50\% of generalization estimates yielding valid upper bounds. Second, our heuristic bound estimation closely matches the true bound estimation and achieves approximately 90\% validity across both linear and nonlinear settings. Finally, although the bound can be infinite in theory, in practice it remains finite, with a 95th percentile value of 0.15 in both selection scenarios.

\begin{table}[ht!]
\centering
\footnotesize
\setlength{\tabcolsep}{4pt}
\renewcommand{\arraystretch}{0.95}
\begin{adjustbox}{max width=\columnwidth}
\begin{tabular}{@{}lrccc@{}}
\toprule
& \multicolumn{4}{c}{$\hat{R}_Q - R_Q$} \\
\cmidrule(lr){2-5}
\textbf{} & \textbf{$\ \ \ \mu \pm \sigma \ \ \ $} & \textbf{Validity} & \textbf{(0.05,\,0.95)} & \textbf{$\boldsymbol{d}_{\text{eff}}$} \\
\midrule
UB (true $U^*$)          & $0.05 \pm 0.06$ & $0.91$ & $(-0.03, 0.15)$ & $5.7$ \\
UB (heuristic $U^*$)     & $0.05 \pm 0.05$       & $0.89$ & $(-0.04, 0.14)$ & $5.2$ \\
Naive IPPW of $\tilde{U}$& $-0.03 \pm 0.04$      & $0.31$ & $(-0.11, 0.04)$ & $2.5$ \\
Raking                   & $0.00 \pm 0.03$      & $0.63$ & $(-0.05, 0.03)$ & $4.7$ \\
Calibration              & $0.00 \pm 0.03$      & $0.63$ & $(-0.05, 0.03)$ & $4.7$ \\
Ent. Balancing        & $0.00 \pm 0.02$      & $0.63$ & $(-0.05, 0.03)$ & $4.7$ \\
\bottomrule
\end{tabular}
\end{adjustbox} \caption{Bound error $\hat{R}_Q - R_Q$ summary metrics on synthetic tasks with linear selection mechanism. Results are shown across 30 tasks and 20 seeds.}
\label{tab:bound-err-synth-lin}
\end{table}

In Figure \ref{appendix:fig:udim-synth}, we show that, in experiments where $X$ and $U$ are binary, our bound is consistently valid ($\hat{R}_Q \geq R_Q)$ and non-vacuous ($\hat{R}_Q - R_Q \leq 0.15)$ across varying observability of $U$. In contrast, IPPW tends to underestimate $R_Q$, with underestimation worsening as the level of $U$ observability decreases. When the selection mechanism is nonlinear, the other baselines also slightly underestimate $R_Q$.

\subsubsection{Extension to continuous data}\label{appendix:continuous_results}
In Figure \ref{appendix:fig:udim-synth-unif}, we show that, in experiments where $U$ are continuous, uniformly distributed random variables, our bound with true $U^*$ is, on average, valid, and tends to underestimate $R_Q$ less than IPPW, calibration, and entropy balancing. When the $U^*$ are nominated by the heuristic, however, our bound performance worsens, and underestimates only slightly less than IPPW. In the continuous data setting, raking tends to perform poorly and exhibits high variance because it either requires binning continuous variables or will drop them completely.

In Figure \ref{appendix:fig:udim-synth-norm}, we show that, when $U$ are continuous normally distributed variables, our bound tends to underestimate $R_Q$ as much as the baseline methods. Across these experiments, we found that accurate density estimation was particularly challenging, with the estimation procedure often producing extremely large density estimates for certain data points. These large densities would dominate in the bound expression, thus skewing the estimate of $R_Q$ and leading to underestimation. Because the results were promising for uniformly distributed data, however, we believe our proposed bound still has potential utility in the continuous data setting. Future work should examine better density estimation methods and should explore modifications to our bound implementation that circumvent the need for direct density estimation.

\begin{figure}[ht!]
    \centering
\includegraphics[width=\linewidth]{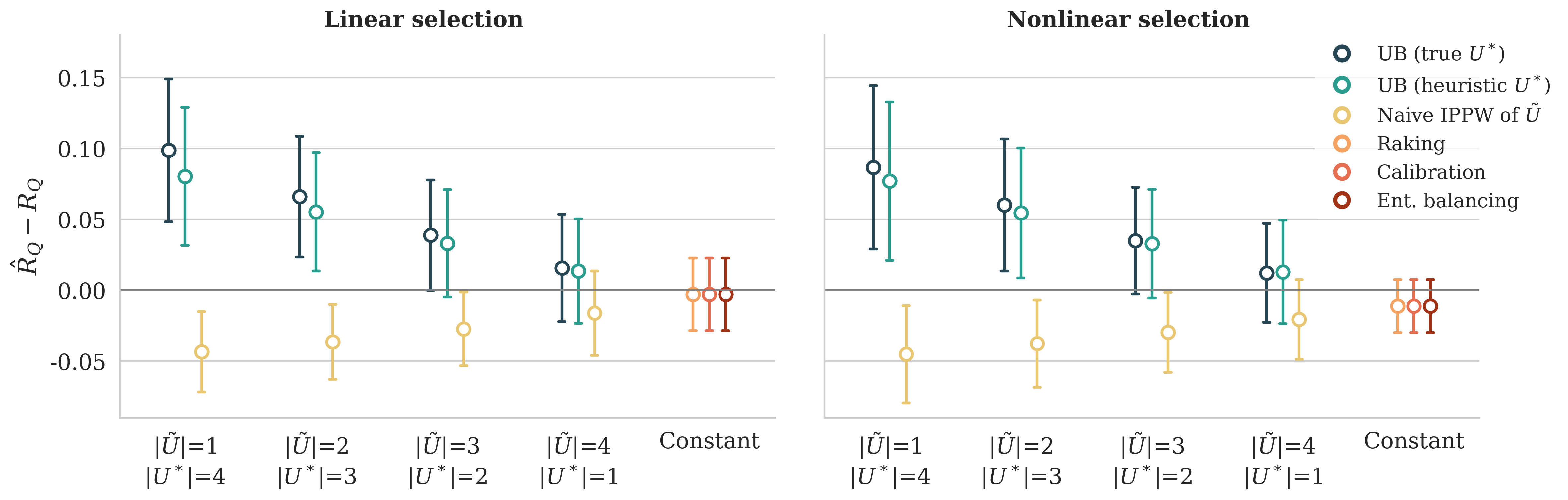}
    \caption{Bound error for different partitions of $U$ and $\tilde{U}$ in synthetic experiments, where $U$ and $X$ are binary. Results are shown across 30 tasks and 5 seeds. }
    \label{appendix:fig:udim-synth}
\end{figure}

\begin{figure}[ht!]
    \centering
\includegraphics[width=\linewidth]{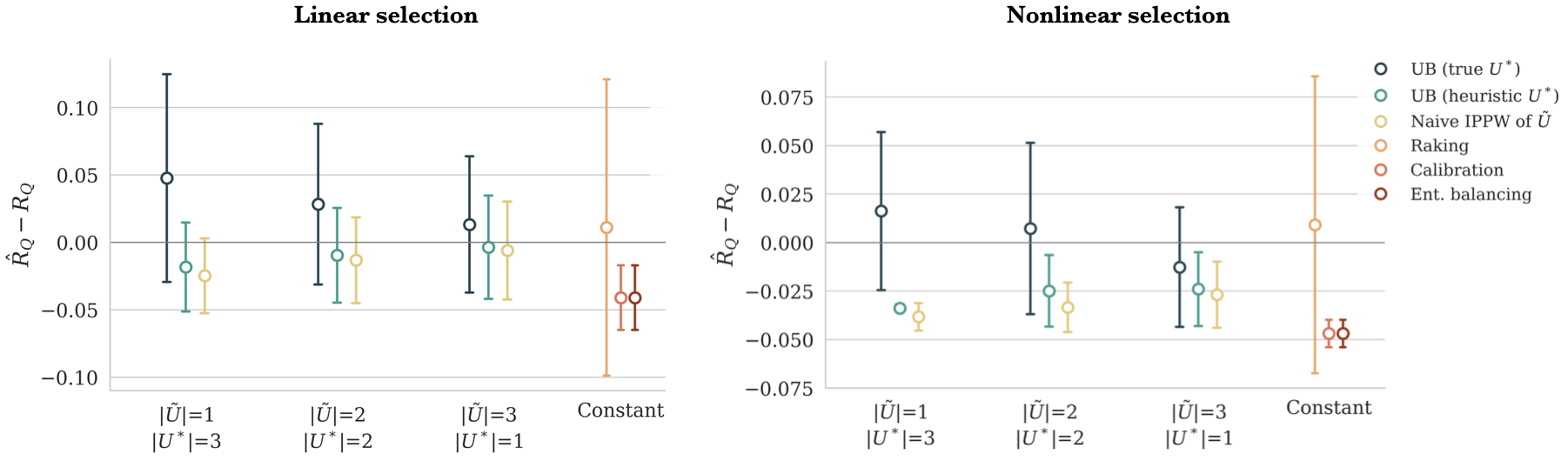}
    \caption{Bound error for different partitions of $U$ and $\tilde{U}$ in synthetic experiments, where $U \overset{\mathrm{iid}}{\sim} \text{Uniform}(0, 1)$ and $X$ are continuous. Results are shown across 3 tasks and 3 seeds. }
    \label{appendix:fig:udim-synth-unif}
\end{figure}

\begin{figure}[ht!]
    \centering
\includegraphics[width=\linewidth]{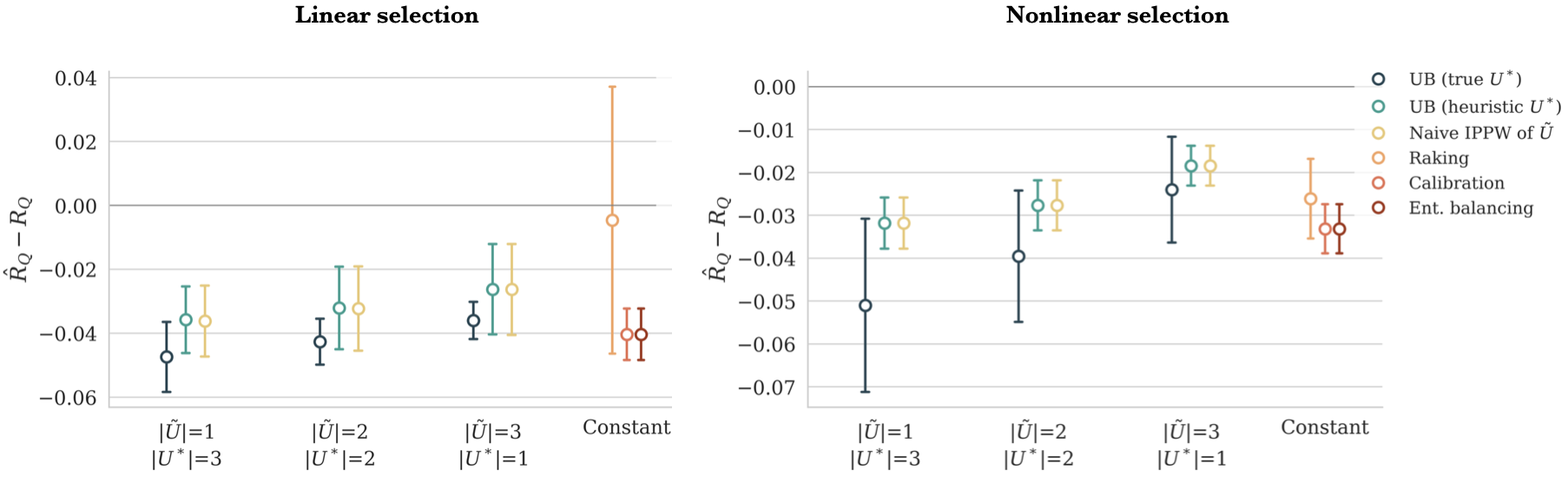}
    \caption{Bound error for different partitions of $U$ and $\tilde{U}$ in synthetic experiments, where $U \sim N(0, \mathbf{\Sigma})$ and $X$ are continuous. Results are shown across 3 tasks and 3 seeds. }
    \label{appendix:fig:udim-synth-norm}
\end{figure}

\subsubsection{Extension to high-dimensional discrete data}\label{appendix:high_dim}
To assess the robustness and scalability of our approach beyond the low-dimensional regime, we conduct an additional ablation study on synthetic high-dimensional discrete data. Specifically, we consider a synthetic dataset with $\text{dim}(X)=60$ observed variables, and $\text{dim}(\tilde{X}) = 20$ covariates used for prediction. We evaluate performance across 7 downstream tasks and 2 random seeds.

We report our results in Table \ref{table:high_dim}. Despite the increased dimensionality, our method maintains a tight bound to the true $R_Q$ and high validity.
\begin{table}[ht!]
\centering
\begin{tabular}{@{}lccc@{}}
\toprule
& \multicolumn{3}{c}{$\hat{R}_Q - R_Q$} \\
\cmidrule(lr){2-4}
\textbf{} & \textbf{$\ \ \ \mu \pm \sigma \ \ \ $} & \textbf{Validity} & \textbf{(0.05,\,0.95)}\\
\midrule
UB (true $U^*$)        & $0.01 \pm 0.02$  & 0.93 & $(0.01,\;0.04)$ \\
UB (heuristic $U^*$)   & $0.03 \pm 0.02$  & 0.99 & $(0.00,\;0.06)$ \\
Naive IPPW of $\tilde{U}$ & $-0.01 \pm 0.02$ & 0.62 & $(-0.06,\;0.01)$ \\
Calibration            & $-0.01 \pm 0.00$ & 0.88 & $(-0.01,\;0.00)$ \\
\bottomrule
\end{tabular}
\caption{Bound error $\hat{R}_Q - R_Q$ on high-dimensional discrete synthetic data, aggregated over 7 tasks and 2 random seeds.}
\label{table:high_dim}
\end{table}

\subsubsection{Ablation on oracle baselines}\label{appendix:other_baselines}
We consider five commonly used importance-weighting methods that are state-of-the-art, but make unrealistic assumptions on a fully observed target distribution $Q$. Specifically, we look at KLIEP \cite{kliep}, KMM \cite{kmmog,kmm}, logistic regression classification \cite{da2,clf1}, RuLSIF \cite{liu2013change}, and uLSIF \cite{kanamori2009least}. All methods are implemented using polynomial kernels where applicable, with hyperparameters selected following standard practice. Because these approaches have direct access to $Q$, they should be interpreted as oracle baselines.

We report results across three experimental settings already studied: first, synthetic binary data with nonlinear selection; second, synthetic continuous data; and third, high-dimensional binary data. We present the results in Table \ref{tab:oracle_combined}. Overall, the results show that our method remains competitive with oracle baselines.

\begin{table}[t]
\centering
\caption{Oracle importance-weighting baselines across three experimental settings.}
\label{tab:oracle_combined}
\begin{tabular}{llccc}
\toprule
\textbf{Experiment} & \textbf{Oracle Method} & $\mu \pm \sigma$ & \textbf{Validity} & $(0.05, 0.95)$ \\
\midrule
{Synthetic Binary}
& Classification & $-0.01 \pm 0.02$ & 0.40 & $(-0.04,\;0.01)$ \\
& KMM (poly)     & $0.01 \pm 0.02$  & 0.90 & $(-0.01,\;0.05)$ \\
& KLIEP (poly)   & $0.01 \pm 0.07$  & 0.20 & $(-0.08,\;0.00)$ \\
& ULSIF (poly)   & $0.01 \pm 0.06$  & 0.20 & $(-0.09,\;0.00)$ \\
& RULSIF (poly)  & $0.01 \pm 0.07$  & 0.10 & $(-0.09,\;0.00)$ \\
\midrule
{Synthetic Continuous}
& Classification & $0.03 \pm 0.08$  & 0.67 & $(-0.05,\;0.17)$ \\
& KMM (poly)     & $-0.03 \pm 0.07$ & 0.53 & $(-0.14,\;0.05)$ \\
& KLIEP (poly)   & $0.01 \pm 0.07$  & 0.67 & $(-0.05,\;0.12)$ \\
& ULSIF (poly)   & $0.01 \pm 0.06$  & 0.67 & $(-0.06,\;0.11)$ \\
& RULSIF (poly)  & $0.01 \pm 0.07$  & 0.60 & $(-0.06,\;0.11)$ \\
\midrule
{High-Dimensional Binary}
& Classification & $0.00 \pm 0.01$  & 0.80 & $(-0.02,\;0.01)$ \\
& KMM (poly)     & $-0.01 \pm 0.01$ & 0.70 & $(-0.02,\;0.01)$ \\
& KLIEP (poly)   & $-0.01 \pm 0.03$ & 0.60 & $(-0.05,\;0.02)$ \\
& ULSIF (poly)   & $-0.01 \pm 0.03$ & 0.60 & $(-0.06,\;0.02)$ \\
& RULSIF (poly)  & $-0.01 \pm 0.03$ & 0.60 & $(-0.06,\;0.01)$ \\
\bottomrule
\end{tabular}
\end{table}

\subsection{Additional results for All of Us experiments}

In Table \ref{appendix:tab:bound-err-aou} we report summary metrics across All of Us for the linear selection mechanism; results for the nonlinear selection mechanism are reported in the main paper in Table \ref{tab:bound-err-aou}. Details on how these metrics were computed is available in Section \ref{appendix:synth-add-res}. Overall, we observe similar performance across nonlinear and linear selection mechanism. In both scenarios, our method has a positive average bound error (0.01 and 0.06 in our heuristic estimation for linear and nonlinear selection, respectively) compared to the potential negative error of baselines. Our method also demonstrates higher rates of valid upper bound estimation ($91\%$ in linear selection) versus baselines ($80\%$ for entropy balancing). However, we see a natural bias-variance tradeoff, where our method has a larger estimate standard deviation and design effect. 

In Figure \ref{appendix:fig:udim-real} we show that our bound is generally valid ($\hat{R}_Q - R_Q > 0$) and non-vacuous ($\hat{R}_Q - R_Q < 0.13$) across varying levels of observability of $U$, both for linear and nonlinear selection mechanisms. Validity and tightness hold both when $U^*$ is the true set of remaining selection variables and when $U^*$ are chosen via our heuristic. In contrast, the baseline methods tend to slightly underestimate the true $R_Q$.

Finally, for better insight into per-task performance, we show five additional tasks chosen at random in Figure \ref{fig:tasks2}, with the specific variables and selection mechanism for each task detailed in the accompanying table. All experiments were run with linear selection. These results highlight that across observability assumptions, our upper bound safely overestimates the true generalization gap versus the risk of underestimation of baselines. Echoing our conclusions in Figure \ref{fig:perf_tasks}, limited information of $U$ contained in $\tilde{U}$, as in the third task $Y=$ HIV with $\dim(\tilde{U})$=1, can lead to difficulty estimating $R_Q$. Similarly, when the unknown drivers of selection $U^*$ are highly correlated with the known variables $\tilde{U}$, generalization gap estimation can improve -- for example, in the first task, where $Y=$Drug dependence, $U^*=$\{Employment status\} is strongly correlated with variables in $\tilde{U}$, resulting in a tighter bound.

\begin{figure}[ht!]
\centering
\includegraphics[width=\columnwidth]{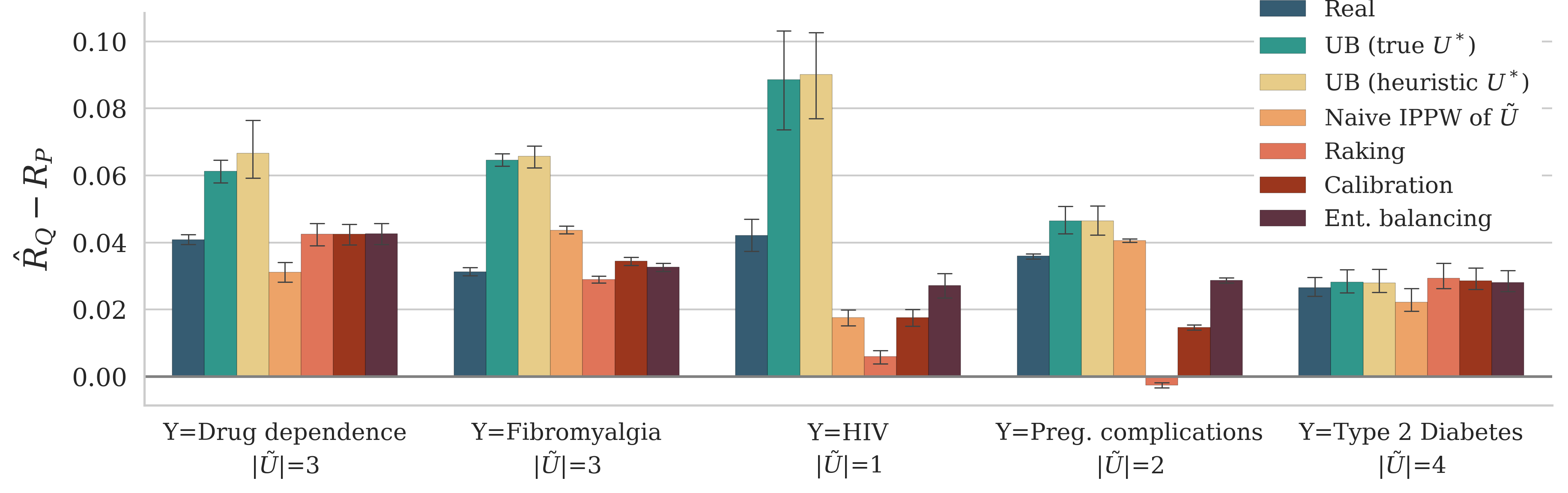}

\vspace{4pt} 

\footnotesize
\setlength{\tabcolsep}{4pt}      
\renewcommand{\arraystretch}{0.95} 
\begin{adjustbox}{max width=\columnwidth}
\begin{tabular}{@{}l
    >{\raggedright\arraybackslash}p{0.19\columnwidth}
    >{\raggedright\arraybackslash}p{0.19\columnwidth}
    >{\raggedright\arraybackslash}p{0.19\columnwidth}
    >{\raggedright\arraybackslash}p{0.19\columnwidth}
    >{\raggedright\arraybackslash}p{0.19\columnwidth}@{}}
\toprule
& \textbf{Y = Drug dependence} & \textbf{Y = Fibromyalgia} & \textbf{Y=HIV} & \textbf{Y = Pregnancy complications} & \textbf{Y = Type II diabetes} \\
\midrule
\textbf{Covariates $\tilde{X}$} & Age, BMI, tobacco usage, vape usage, drink frequency, vitals & Age, BMI, tobacco usage, drink frequency, vitals & Age, BMI, tobacco usage, drink frequency, vitals & Age, BMI, drink frequency, vitals & Age, BMI, tobacco usage, drink frequency, vitals \\
\textbf{Known $\tilde{U}$}          & Education level (+), income (+)  & Education level (+), age (+), race ($\uparrow$ White)     &  Age (+) & Age (+), income (+) & Education level (+), confidence in medical system (+), has insurance (+), education level (+)  \\
\textbf{Unknown $U^*$}          &   Employment status (+) &   General health (-)  & Education level (+), has housing (+), employment status (+) & Education level (+), confidence in medical system (+), general health (+)    & Income (+) \\
\bottomrule
\end{tabular}
\end{adjustbox}\caption{Generalization gap $\hat{R}_Q - R_P$ on All of Us experiments. Selection weights were designed to reflect known patterns of EHR selection bias that have been documented in the literature. (+) and (-) denotes we induced over- and under-sampling, respectively. Results are shown across 20 seeds.} \label{fig:tasks2}
\end{figure}

\begin{table}[ht!]
\centering

\footnotesize
\setlength{\tabcolsep}{4pt}
\renewcommand{\arraystretch}{0.95}

\begin{adjustbox}{max width=\columnwidth}
\begin{tabular}{@{}lrccc@{}}
\toprule
& \multicolumn{4}{c}{$\hat{R}_Q - R_Q$} \\
\cmidrule(lr){2-5}
\textbf{} & \textbf{$\ \ \ \mu \pm \sigma \ \ \ $} & \textbf{\% Valid} & \textbf{(0.05,\,0.95)} & \textbf{$\boldsymbol{d}_{\text{eff}}$} \\
\midrule
UB (true $U^*$)          & $0.02 \pm 0.03$ & $0.91$ & $(-0.02, 0.09)$ & $18.2$ \\
UB (heuristic $U^*$)     & $0.01 \pm 0.03$       & $0.91$ & $(-0.01, 0.07)$ & $7.4$ \\
Naive $\tilde{U}$ of IPPW& $-0.01 \pm 0.02$      & $0.54$ & $(-0.04, 0.01)$ & $1.5$ \\
Raking                   & $-0.01 \pm 0.02$      & $0.67$ & $(-0.05, 0.00)$ & $1.6$ \\
Calibration                   & $-0.01 \pm 0.02$      & $0.67$ & $(-0.03, 0.00)$ & $1.6$ \\
Ent. Balancing                    & $-0.01 \pm 0.01$      & $0.80$ & $(-0.03, 0.00)$ & $3.1$ \\
\bottomrule
\end{tabular}
\end{adjustbox}\caption{Bound error $\hat{R}_Q - R_Q$ summary metrics on All of Us tasks with linear selection mechanism. Results are shown across 30 tasks and 20 seeds.}
\label{appendix:tab:bound-err-aou}
\end{table}

\begin{figure}[ht!]
    \centering
\includegraphics[width=\linewidth]{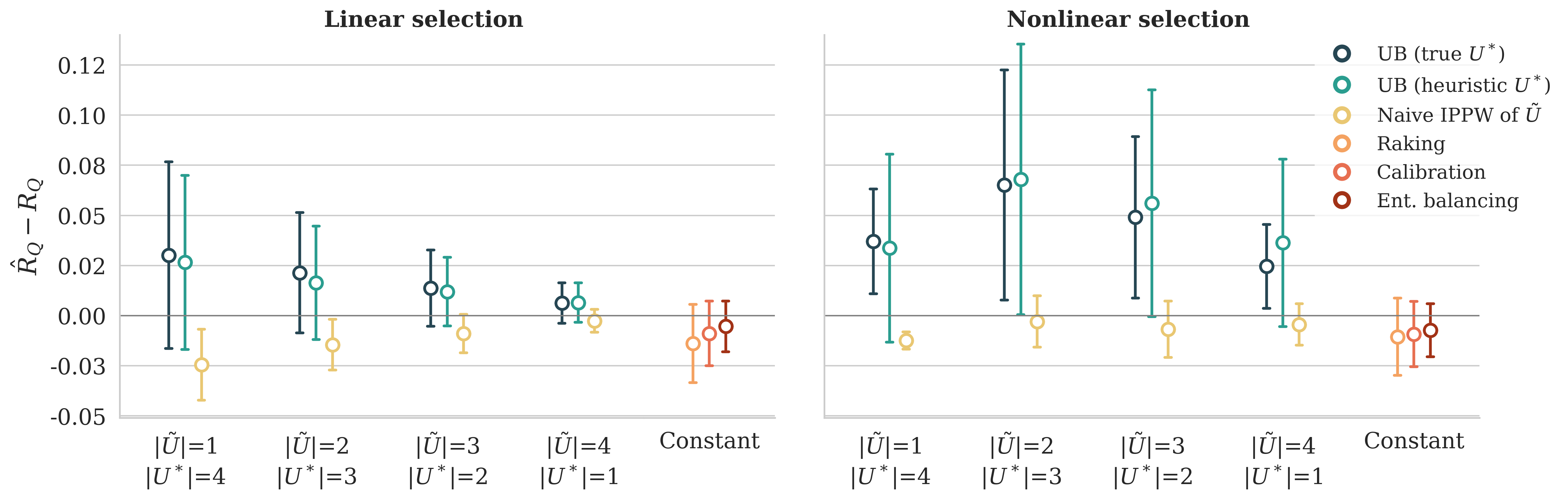}
    \caption{Bound error for different partitions of $U$ and $\tilde{U}$ in All of Us experiments. Results are shown across 30 tasks and 20 seeds.}
    \label{appendix:fig:udim-real}
\end{figure}

\subsection{Checking for $U^*$ Heuristic Correctness}\label{appendix:heuristic_exp}

\subsubsection{Experimental Setup}
We test for the accuracy of our heuristic in nominating the correct selection variables $U^*$.

Given a desired set $A$ of items, a predicted set $B$ of items, and a size $N$ of total items in our set space, we define the number of \textbf{true positives (TP)} as $|A \cap B|$; \textbf{true negatives (TN)} as $N - |A \cup B|$; \textbf{false positives (FP)} as $B - A$; and \textbf{false negatives (FN)} as $A - B$. Given these definitions, we can compute F1, precision, and recall as our performance metrics, in addition to the Jaccard index. 

In our setting, the desired set $A$ are the true variables in $U^*$ and the set $B$ is the nominated set of variables $\hat{U}^*$ identified by our heuristic. The set space of available items are the variables in $S := (X, U^*)$

Our heuristic algorithm is described in Section \ref{sec:methods:u_pred}. We test against three heuristic baselines. \textbf{Max corr} nominates the $k$ most correlated variables to the observed $\tilde{U}$, based on average Spearman correlation. \textbf{Random} nominates $k$ random variables from $S$. For both of these, given our known $\dim(U) = 5$, we pick $k \in [2,4]$. Finally, we also select variables based on their \textbf{Cohen's $d$} statistic which leverages the known summary statistics. Similar to the process described in Appendix \ref{appendix:method-extensions:heuristic} for variable filtering, we rank each variable $V \in (X,U^*)$ by its statistic $d(V)$. Let $\delta(\tilde{U})=\min(\{d(\tilde{U}_j\}_{j=1}^{\dim(\tilde{U})})$ be the minimum statistic for observed $\tilde{U}$. Then we set $\hat{U}^* = \{ V \in (X, U^*) : \text{d(V)} \geq \delta(\tilde{U})\}$.

To reasonably test the accuracy of our heuristic across a variety of scenarios, we collect the average of the four metrics (F1, precision, recall, and Jaccard) over several data generation processes. Specifically, for the synthetic dataset we run over the grid of dim$(X_U) \in [4,15]$,  dim$(X_C) \in [4,15]$, $U$ correlation $R_{i,j} \in [0,0.25,0.5], d_S \in [1,3]$. We fix $|\mathcal{D}_Q|=1e4, \, \text{dim}(U) = 5, \, \text{dim}(\tilde{X}) = 4, \, d_X = d_Y = 2, \sigma^*_{\text{logit}(S)} = \sigma^*_{\text{logit}(Y)}  = \sigma^*_{\text{logit}(X)}  = 2$. For the All of Us experiment,  we run over dim$(\tilde{X}) \in [7,15]$,  $\sigma^*_{\text{logit}(S)} \in [2, 3]$, $d_S \in [1,3]$. We fix $|\mathcal{D}_Q|=0.3N, \, \text{dim}(U) = 5$. 

Note that in both experiments, we test with both linear and nonlinear selection mechanisms. Using a nonlinear selection mechanism implicitly tests our heuristic's sensitivity to model misspecification, since the calibration equation in our heuristic assumes that selection is linear (i.e., that the logistic link function $g$ is applied to a linear function of $\tilde{U}$, $X$, and $U^*$). 

\subsubsection{Results}

We show additional results for semi-synthetic All of Us in Figure \ref{fig:real_heuristic}. Even across a variety of data settings, including nonlinear selection mechanisms, our proposed heuristic has the highest scores across the four metrics, with an F2 of 0.81 and 0.84 in the All of Us and synthetic settings, respectively. This indicates decent robustness to model misspecification. Finally, we see that, as expected, in the real dataset correct $U^*$ identification becomes more challenging as the number of candidate options increase. 
\begin{figure}[ht!]
    \centering
\includegraphics[width=.7\linewidth]{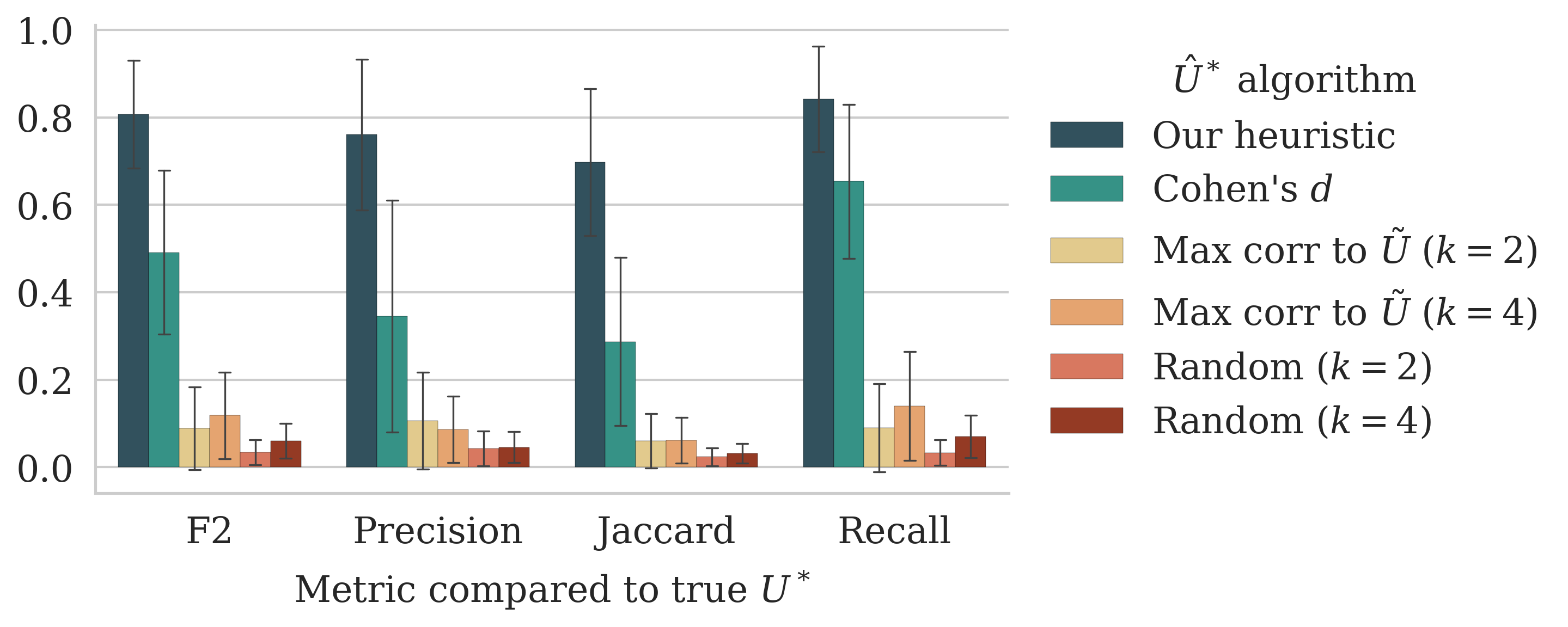}
    \caption{Comparing $U^*$ nomination algorithms in identifying true variables, on All of Us and across various data settings.}
    \label{fig:real_heuristic}
\end{figure}

\subsection{Sensitivity to observed $\tilde{U}$}\label{appendix:sens_utilde_scatter}

\subsubsection{Experimental Setup}
We test our proposed bound's sensitivity to the ``observability" of $U$ from $\tilde{U}$ in two ways: (i) what portion of $U$ is observed, i.e. what is the dimension of $\tilde{U}$ versus $U^*$ and (ii) how much information on selection is contained in $\tilde{U}$, or, in other words, to what extent does $\tilde{U}$ predict selection? Results for (i), which compare the generalization error $\hat{R}_Q - R_Q$ across the dimension of $U^*$, are included prior sections. In summary, these results highlight that our method's estimation of $\hat{R}_Q$ improves given more information on $\tilde{U}$, although even with limited information (i.e. dim($\tilde{U}$)=1) our method still performs well. In this section, we focus on our approach and results for (ii).

To measure the extent to which $\tilde{U}$ predicts selection, we fit a classifier to estimate $p(S = 1 \mid \tilde{U})$. We then examine the logloss of the fitted classifier, relative to the best case logloss of the true selection model $p(S = 1 \mid U)$. Specifically, we compute the relative logloss of $S \sim \tilde{U}$ as $\mathbb{E}_Q[\ell(p_{\phi}(U),S)] - \mathbb{E}_Q[\ell(p_{\theta}(\tilde{U}),S)]$ where $p_\phi, p_\theta$ are the models fit to predict $S$ given $U$ and $\tilde{U}$, respectively. The intuition is that the greater the logloss, the less predictive $\tilde{U}$ is of $S$. Thus we hypothesize that logloss is likely to be correlated with greater bound error and additionally report the Pearson correlation between the two. We report both the true bound error (using the true $U^*$) and the bound error provided by the heuristic $\hat{U}^*$. 

\subsubsection{Results}
In Figure \ref{appendix:fig:ures_synth} and Figure \ref{appendix:fig:ures_real}, we show the relationship between the bound error and the logloss in fully synthetic and All of Us data. From these plots, we observe that there is mild (Pearson's $\rho \in [0.14, 0.26]$), but statistically significant (p-value$ < 0.001$), correlation between logloss and bound error. While more empirical and theoretical work is needed, these results suggest that, in the real-world where $R_Q$ is unknown, the logloss of a fitted $S$-on-$\tilde{U}$ classifier can potentially serve as a diagnostic indicator of our proposed bound's tightness and validity. 

\begin{figure}[ht!]
    \centering
\includegraphics[width=.8\linewidth]{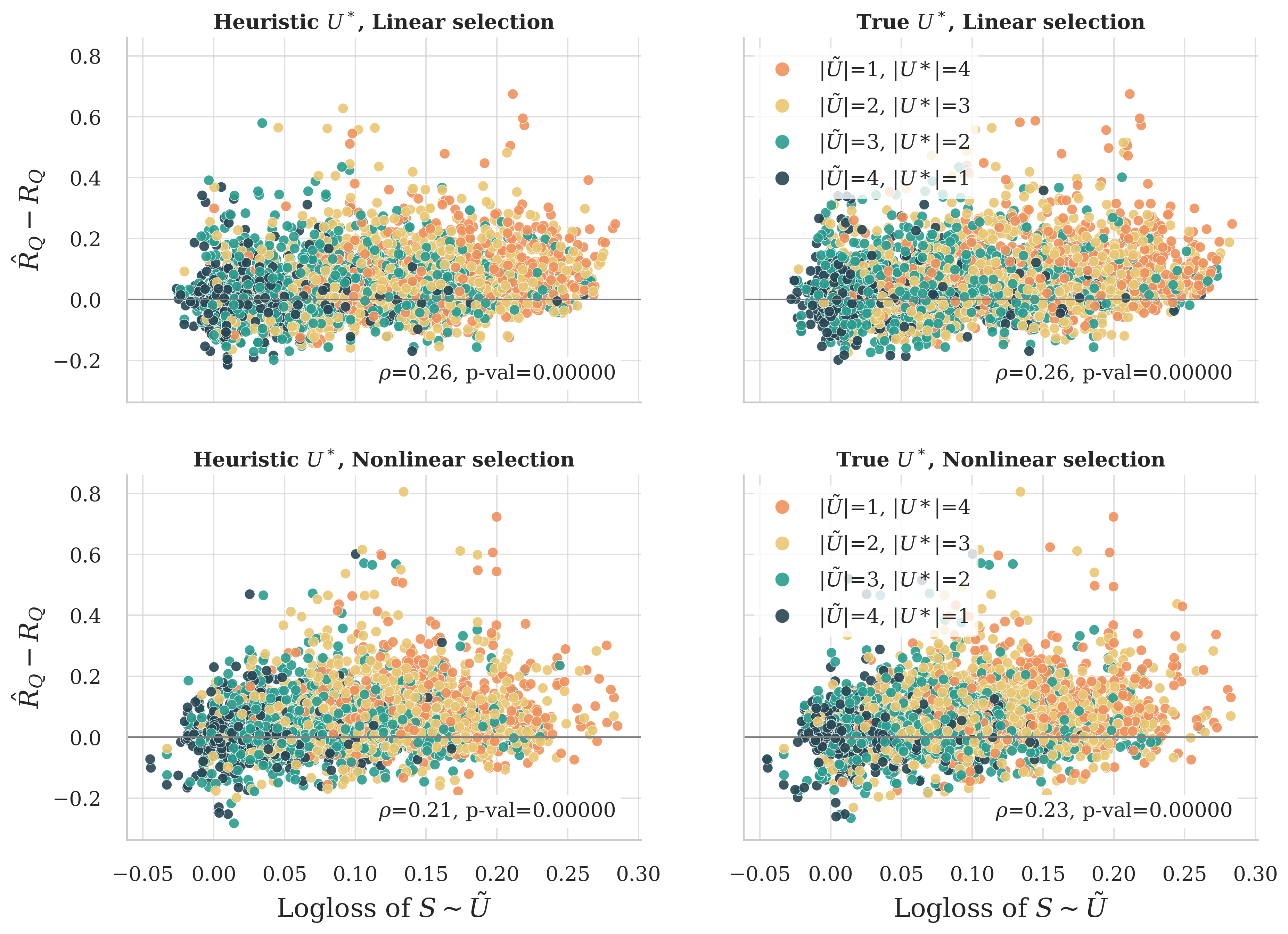}
    \caption{Bound error versus logloss of a classifier fit to predict $p(S = 1 \mid \tilde{U})$, on fully synthetic data.}
    \label{appendix:fig:ures_synth}
\end{figure}

\begin{figure}[ht!]
    \centering
\includegraphics[width=\linewidth]{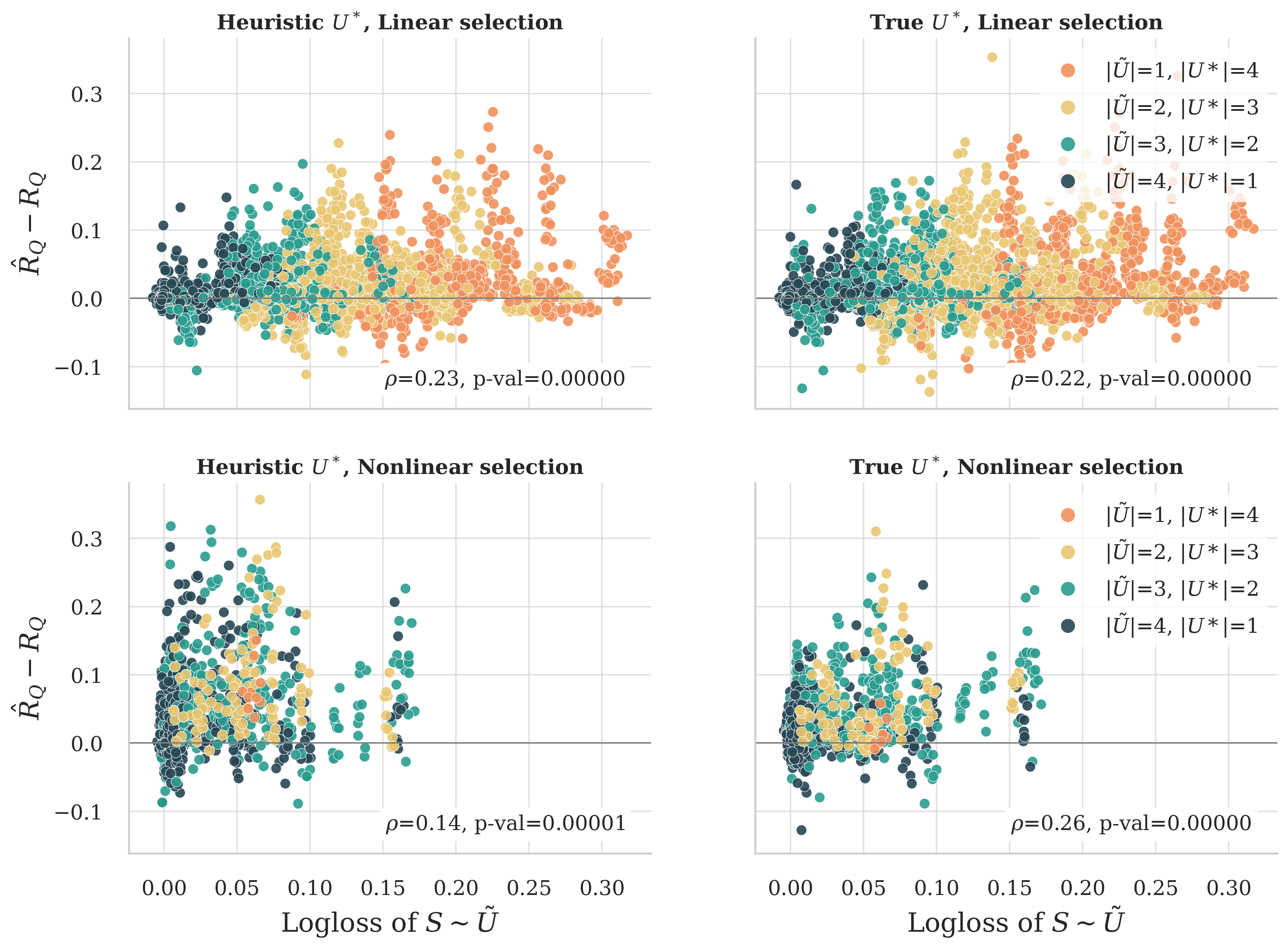}
    \caption{Bound error versus logloss of a classifier fit to predict $p(S = 1 \mid \tilde{U})$, on All of Us data.}
    \label{appendix:fig:ures_real}
\end{figure}

\subsection{Sensitivity to assumption violations}\label{appendix:assmp_viol_oracle}
\subsubsection{Experimental Setup}

As detailed in Appendix \ref{appendix:proofs:decomp}, we  can decompose our method's bound error into three terms: 
\begin{align*}
    \hat{R}_Q - R_Q &= \Delta_{\text{TBE}}
    + \Delta_{\text{CI}} 
    + \Delta_{\text{CS}}
\end{align*}
of the true bound error ($\Delta_{\text{TBE}}$), error arising from violating conditional independence ($\Delta_{\text{CI}}$), and error arising form violating common support ($\Delta_{\text{CS}}$). 

Although our upper bound is in theory valid, in practice finite samples can cause assumption violations. We test our bound validity by varying four key parameters that control assumption violations:
\begin{enumerate}
    \item We directly control sample size $|\mathcal{D}_Q|$. Decreased sample size is expected to increase assumption violations and also the probability of invalid underestimation. In synthetic data, this is possible by directly controlling $N$. In the All of Us case, we resample with replacement from the selected cohort.  
    \item We control the strength of the selection mechanism by varying the standard deviation of the probability of selection, $\sigma^*_{\text{logit}(S)}$. This is possible for both synthetic and All of Us data. 
    \item We vary the feature imbalance of $\tilde{X}$ and $Y$. This is motivated by the fact that overlap is more likely to be violated if we lack common support between variables, which is more likely if variables are imbalanced. Furthermore, density estimation, especially in high dimensional settings, is more likely to fail if variables are imbalanced. In synthetic data, we vary the standard deviation of the logits $\sigma^*_{\text{logit}(X)}$, $\sigma^*_{\text{logit}(Y)}$. In All of Us, we vary how many variables to add to a fixed $\tilde{X}$, where variables are added in order of its maximum correlation to the selection variables $\tilde{U}$, with the hypothesis that these variables will be most skewed under selection bias.
    \item We vary the dimension of $\tilde{X}$, where we differentiate between the prior experiment of adding skewed variables only by adding variables that are uncorrelated with $\tilde{U}$. In synthetic data, we vary dim($X_C$) and include these variables in $\tilde{X}$. In All of Us, we vary how many variables to add to a fixed $\tilde{X}$, where variables are added in opposite order to that above, i.e., picking variables with minimal maximum correlation to the selection variables $\tilde{U}$.

\end{enumerate}

We run on synthetic and All of Us data with linear selection mechanisms. The default parameters are the same as described for linear selection mechanisms in prior sections, although to reduce computational complexity, in the All of Us experiments we run on a random subset of 5 tasks and 20 seeds, instead of 30 tasks.
For each parameter setting, we plot the average contribution from each of the three terms with respect to the true generalization bound error $\hat{R}_Q-R_Q$. To compare assumption violations to bound validity, we additionally plot the percent underestimation across all tasks where underestimation is calculated as is the number of tasks with less than -0.01 generalization bound error.

\subsubsection{Results}

We show our results for varying the four parameters in Figures \ref{fig:decomp_synth} and \ref{fig:decomp_real} for the synthetic and All of Us data, respectively. In both settings, as sample size increases, selection strength decreases, $\tilde{X}$ imbalance decreases, or feature dimension decreases, the contribution of the violation assumption terms (conditional independence $\Delta_{ci}$ and common support $\Delta_{cs}$) decreases, indicating that these assumptions are increasingly satisfied. As these terms go to zero, the (positive) theoretical bound error $\Delta_{tbe}$ dominates, leaving the resulting bound $\hat{R}_Q > R_Q$. We see this reflected in the decreasing rates of bound underestimation. 

\begin{figure}[ht!]
    \centering
\includegraphics[width=0.8\linewidth]{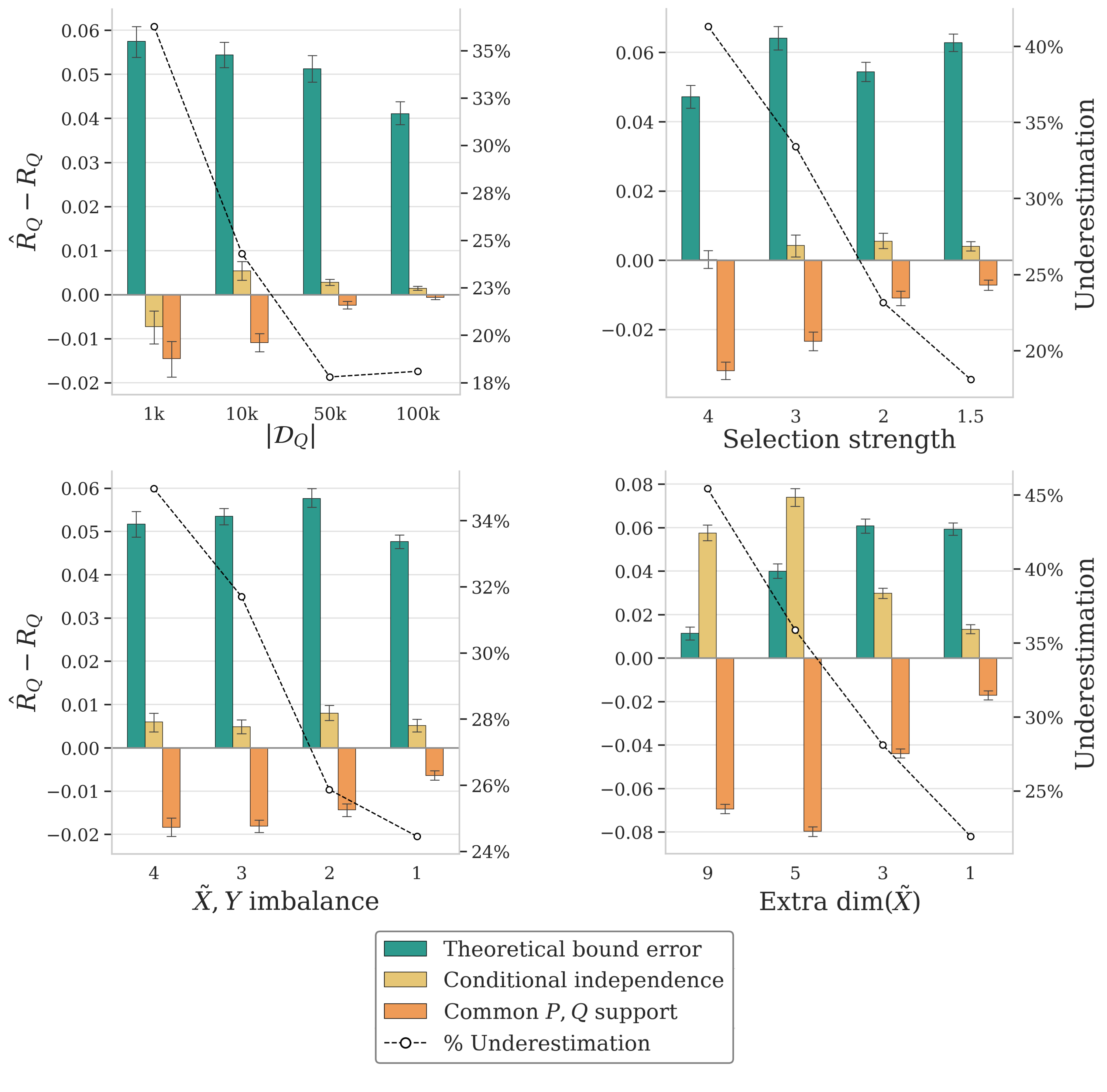}
    \caption{Decomposition of the generalization bound error $\hat{R}_Q-R_Q$, for the true $U^*$, in synthetic data with linear selection mechanism. Each experiment is computed across 20 tasks and 5 seeds.}
    \label{fig:decomp_synth}
\end{figure}

\begin{figure}[ht!]
    \centering
\includegraphics[width=0.8\linewidth]{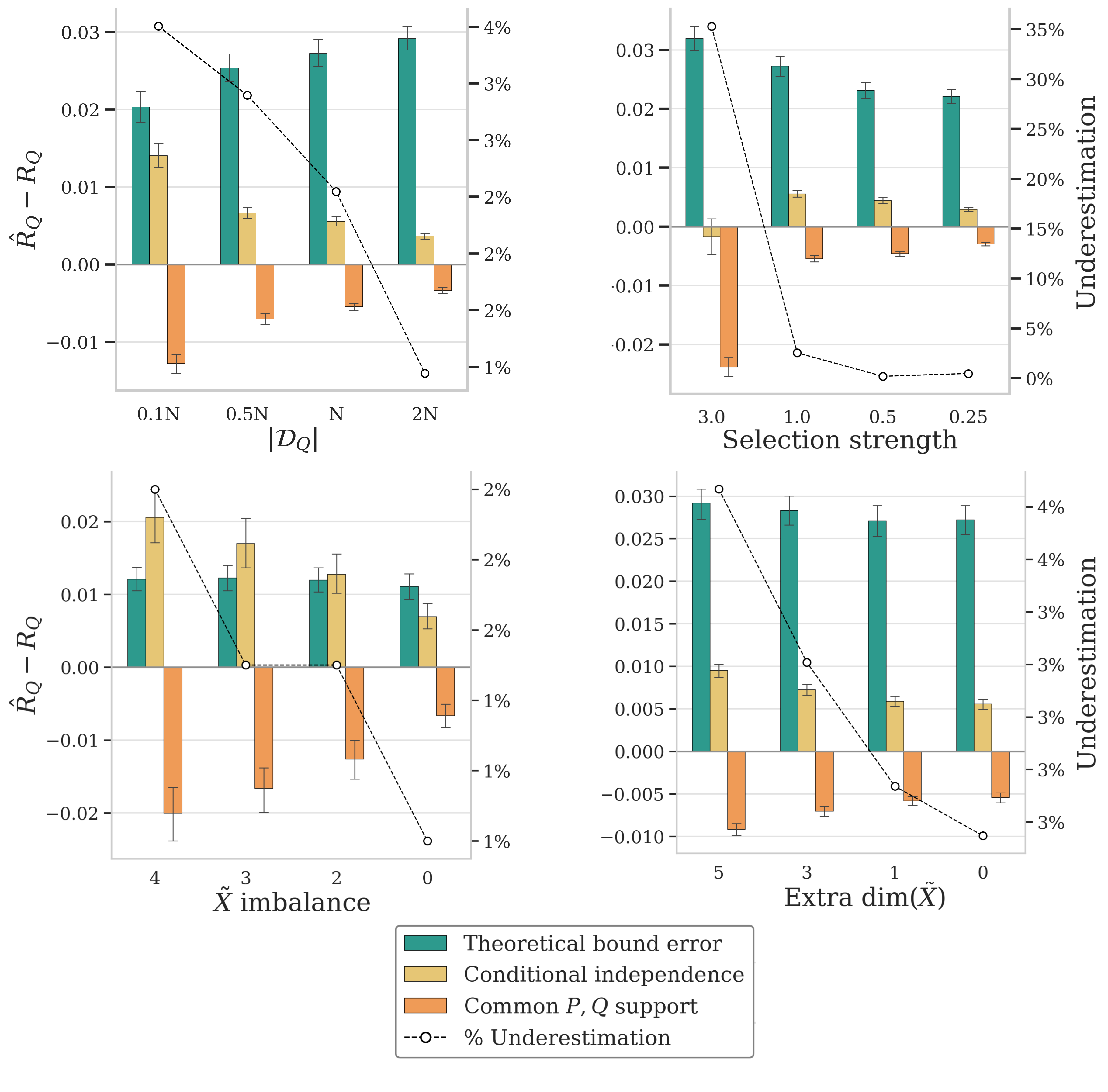}
    \caption{Decomposition of the generalization bound error $\hat{R}_Q-R_Q$, for the true $U^*$, in All of Us data with linear selection mechanism. Each experiment is computed across 20 tasks and 5 seeds.}
    \label{fig:decomp_real}
\end{figure}




\subsection{Application to Political Survey Data}\label{appendix:cses}
\subsubsection{Experimental setup}
To highlight our method’s broader applicability beyond the medical setting, we conducted an additional experiment detecting selection bias in political survey data. Political surveys are frequently constructed from volunteer-based or opt-in samples, making them particularly susceptible to non-random selection mechanisms that can distort population-level inference and undermine external validity. We use the 2024 Cooperative Congressional Election Study (CCES) dataset \cite{schaffner2024cces} which contains rich demographic information, political attitudes, and self-reported voting behavior collected during the 2024 election cycle in the U.S..

Similar to the All of Us dataset setting, we treat the CCES dataset as the target population $\mathcal{D}_Q$ and construct a linear logistic selection mechanism to sample the biased population $\mathcal{D}_P$. We specifically oversample for engagement with the news, oversample for strong party affiliation (for any party), undersample for participants working more than two jobs, and oversample for participants who are parents. Explicitly, for these features, the corresponding weight vector $\beta = [0.3, 0.3, -0.3,0.3]$. Using this biased dataset, we then test for selection bias in a elastic net model predicting voter turnout in the 2024 election, using for the covariates $\tilde{X}$ gender, education, party affiliation, if the participant is urban-dwelling, national area of residence (i.e., if the participant lives in the Midwest), personal belief on inflation, and marital status. Extraneous variables $X \setminus \tilde{X}$ included race, ethnicity, familial immigration status, usage of social media, and personal beliefs on abortion, immigration, the economy, and assault rifles. All data was made discrete.

\subsubsection{Results}
The results in Table \ref{table:cces_results}, run across 20 seeds, demonstrate that our method successfully detects the induced bias, highlighting its broader relevance beyond medical datasets.

\begin{table}[ht!]
\centering
\begin{tabular}{@{}lccc@{}}
\toprule
& \multicolumn{3}{c}{$\hat{R}_Q - R_Q$} \\
\cmidrule(lr){2-4}
\textbf{} & \textbf{$\ \ \ \mu \pm \sigma \ \ \ $} & \textbf{Validity} & \textbf{(0.05,\,0.95)}\\
\midrule
UB (true $U^*$) 
& $0.01 \pm 0.01$ 
& $1.00$ 
& $(0.00,\;0.02)$ \\

UB (heuristic $U^*$) 
& $0.00 \pm 0.00$ 
& $1.00$ 
& $(-0.01,\;0.00)$ \\

Naive IPPW of $\tilde{U}$
& $-0.01 \pm 0.00$ 
& $1.00$ 
& $(-0.01,\;0.00)$ \\

Calibration 
& $-0.03 \pm 0.00$ 
& $0.00$ 
& $(-0.03,\;-0.03)$ \\
\bottomrule
\end{tabular}
\caption{Bound error $\hat{R}_Q - R_Q$ summary metrics for a single CCES task predicting voter turnout with linear selection mechanism, across 20 seeds.}
\label{table:cces_results}
\end{table}


\end{document}